\documentclass[10pt,twocolumn,letterpaper]{article}

\usepackage[]{algorithm2e}
\usepackage{cvpr}
\usepackage{times}
\usepackage{epsfig}
\usepackage{graphicx}
\usepackage{amsmath}
\usepackage{amssymb}
\usepackage{overpic}
\usepackage{rotating,soul}
\usepackage{wrapfig}
\usepackage{sidecap}
\usepackage[T1]{fontenc}
\usepackage{multirow}
\usepackage{amsthm}
\usepackage[config=altsf]{subfig}
\usepackage{bm}
\usepackage{array}
\usepackage{booktabs}
\usepackage{lipsum}

\newtheorem{theorem}{Theorem}
\newtheorem{lemma}{Lemma}
\newtheorem{definition}{Definition}
\newtheorem{condition}{Condition}

\graphicspath{{./images/}}

\usepackage{color}

\newcommand{\tref}[1]{Tab.~\ref{#1}}

\newcommand{\eref}[1]{Eq.~(\ref{#1})}
\newcommand{\Eref}[1]{Equation~(\ref{#1})}
\newcommand{\fref}[1]{Fig.~\ref{#1}}
\newcommand{\Fref}[1]{Figure~\ref{#1}}
\newcommand{\sref}[1]{Sec.~\ref{#1}}
\newcommand{\Sref}[1]{Section~\ref{#1}}
\newcommand{\Ltwo}{\ell_2}

\newcommand{\Var}{\mathrm{Var}}
\newcommand{\var}{\mathrm{var}}
\newcommand{\cov}{\mathrm{cov}}
\newcommand{\E}{\mathrm{E}}
\newcommand{\bz}{\mathbf{Z}}
\newcommand{\bx}{\mathbf{X}}
\newcommand{\by}{\mathbf{Y}}

\newcommand{\bbx}{\mathbf{x}}

\newcommand{\Set}{\mathcal{S}}

\usepackage[pagebackref=true,breaklinks=true,letterpaper=true,colorlinks,bookmarks=false]{hyperref}

\def\cvprPaperID{1746} 

\def\rnum#1{\expandafter{\romannumeral #1}}

\makeatletter

\newcommand{\tblcaption}[1]{\def\@captype{table}\caption{#1}}
\makeatother

\newcolumntype{Y}{>{\centering\arraybackslash}p{1.9zw}}
\cvprfinalcopy
\ifcvprfinal\fi

\begin{document}
\title{Dimensionality's Blessing: Clustering  Images by Underlying Distribution}

\author{
Wen-Yan Lin\\
Advanced Digital Sciences Center\\
{\tt\small linwenyan.daniel@gmail.com}
\\
Jian-Huang Lai\\
Sun Yat-Sen University \\
{\tt\small stsljh@mail.sysu.edu.cn}
\and
Siying Liu\\
Institute of Infocomm Research \\
{\tt\small szewinglau@gmail.com}
\\
Yasuyuki Matsushita\\
Osaka University \\
{\tt\small yasumat@ist.osaka-u.ac.jp	}
}
\maketitle

\begin{abstract}

Many  high dimensional  vector distances tend to a constant. This is typically considered a negative ``contrast-loss''
phenomenon that hinders  clustering and other machine learning techniques. We reinterpret
``contrast-loss'' as a blessing. Re-deriving ``contrast-loss'' using the law of large numbers,
we show    it results in  a distribution's instances concentrating on  a thin ``hyper-shell''. The hollow center means  apparently
chaotically overlapping distributions are actually  intrinsically separable.  We use
 this to develop distribution-clustering, an elegant algorithm for 
 grouping of  
 data points by their (unknown) underlying distribution. 
Distribution-clustering, creates notably clean clusters from  raw unlabeled data, 
 estimates the number of clusters for itself and is inherently robust to ``outliers'' which 
form their own clusters. This enables    trawling for patterns in    unorganized data 
and may be the key to enabling machine intelligence.

\end{abstract}

\section{Introduction}

Who is thy neighbor? The question is  universal  and old as the Bible. In computer vision, images are typically converted into a
high-dimensional vector known as image descriptors. \emph{Neigborness} of images is defined as distances between their respective  descriptors.
This approach has had mixed success. Descriptors excel at nearest-neighbors retrieval applications. However, descriptor distances
are rarely effective in other neighbor based machine learning tasks like  clustering.

\begin{figure}[htp]
\center
\begin{tabular}{cc}
\includegraphics[width=0.45\linewidth]{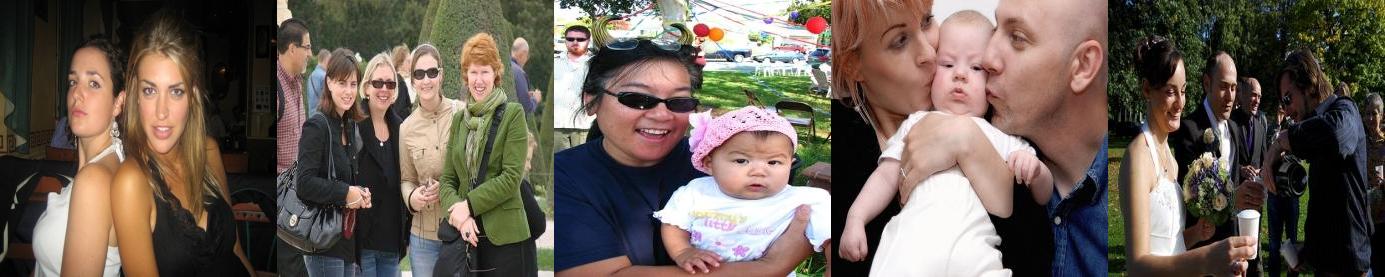}&
\includegraphics[width=0.45\linewidth]{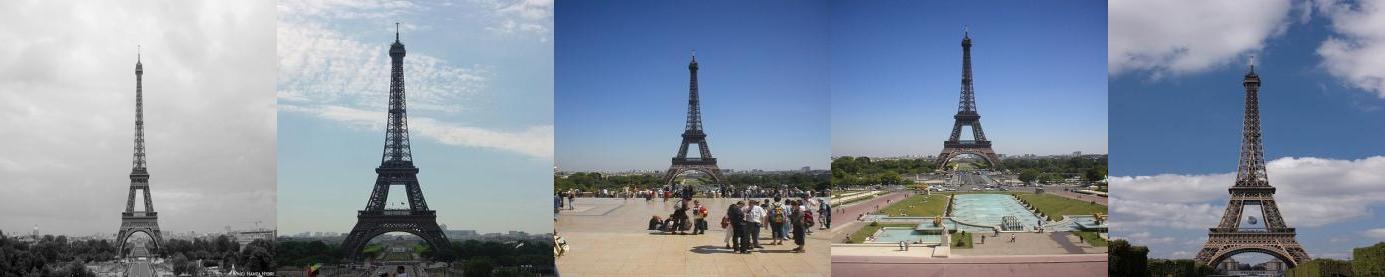}\\
\includegraphics[width=0.45\linewidth]{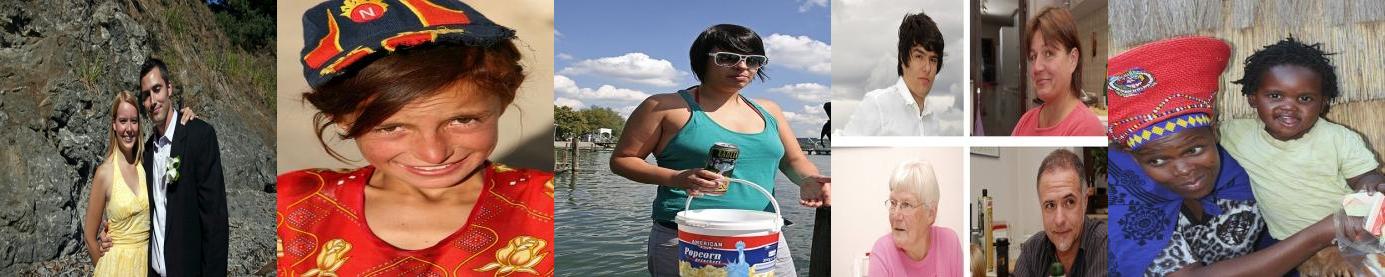}&
\includegraphics[width=0.45\linewidth]{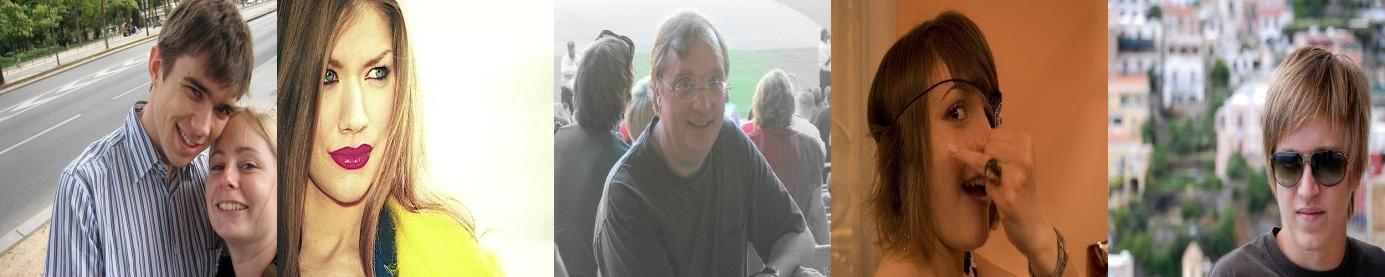}\\
\includegraphics[width=0.45\linewidth]{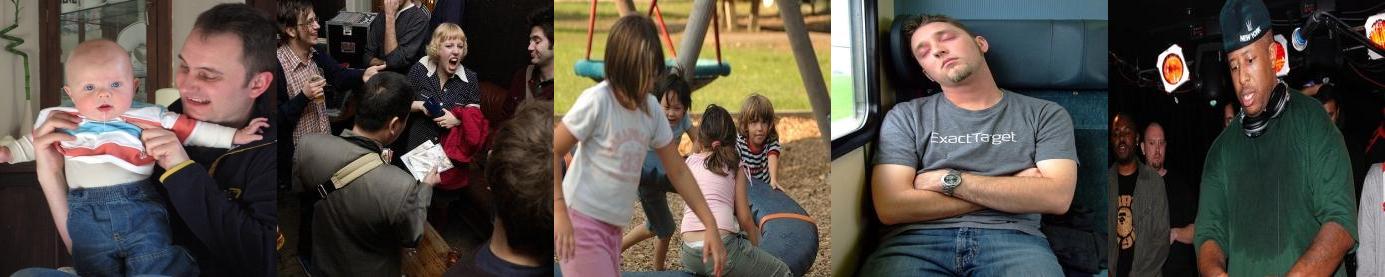}&
\includegraphics[width=0.45\linewidth]{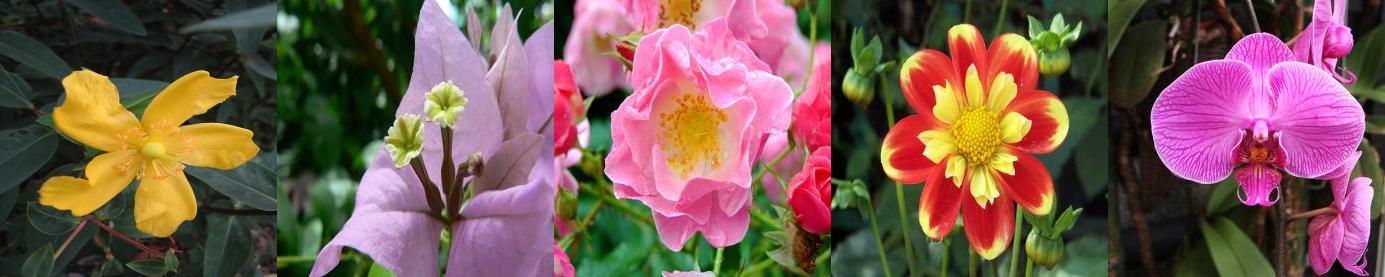}\\
\includegraphics[width=0.45\linewidth]{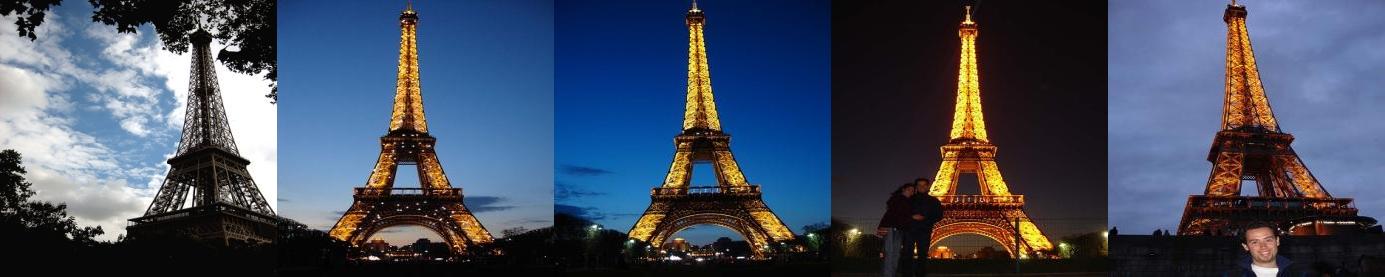}&
\includegraphics[width=0.45\linewidth]{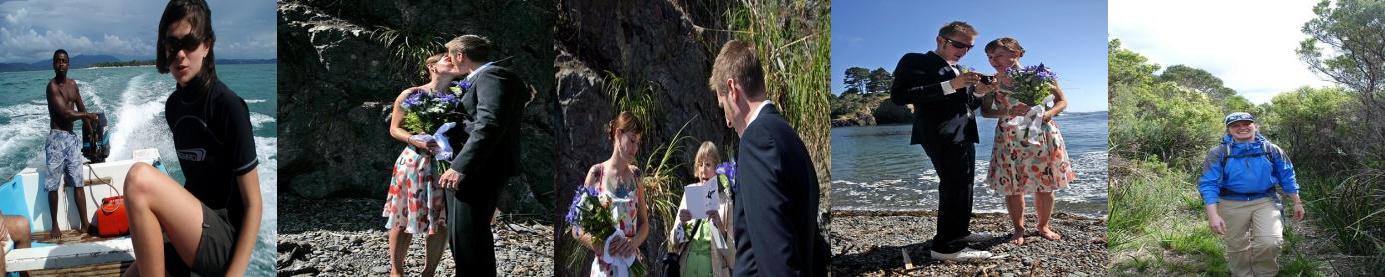}\\
\includegraphics[width=0.45\linewidth]{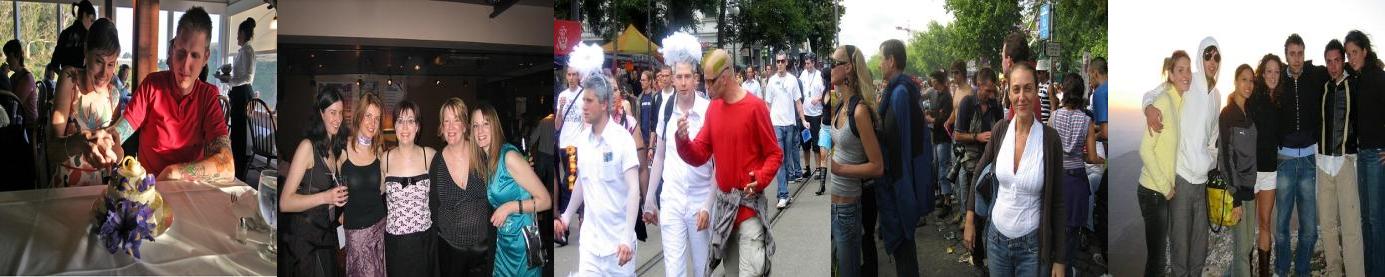}&
\includegraphics[width=0.45\linewidth]{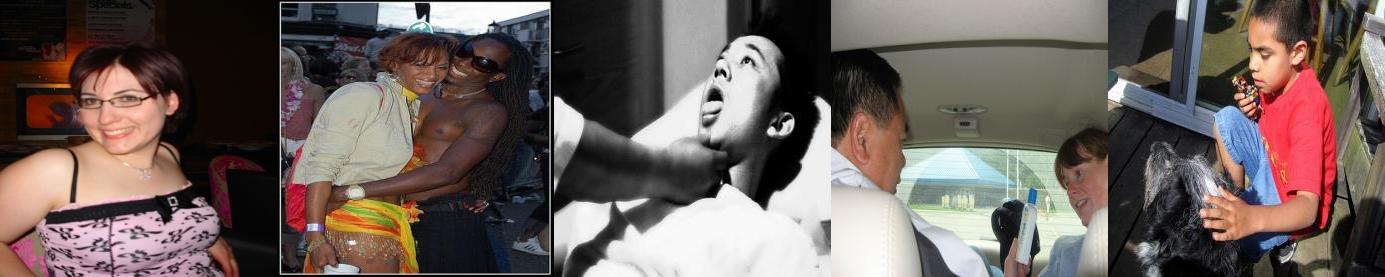}\\
\includegraphics[width=0.45\linewidth]{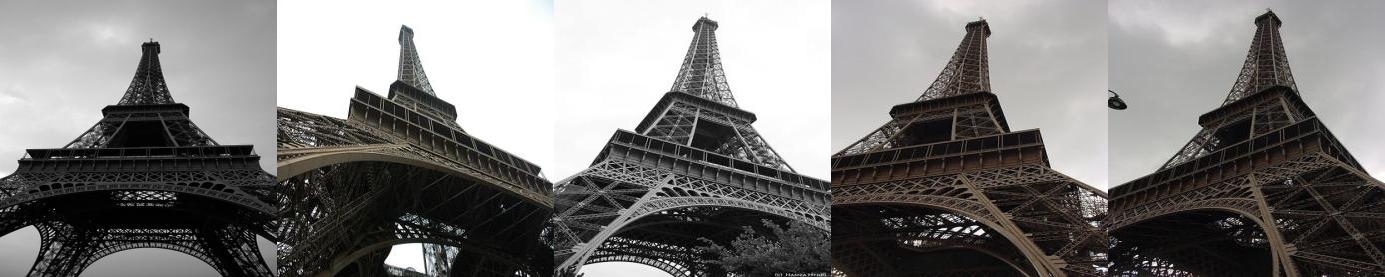}&
\includegraphics[width=0.45\linewidth]{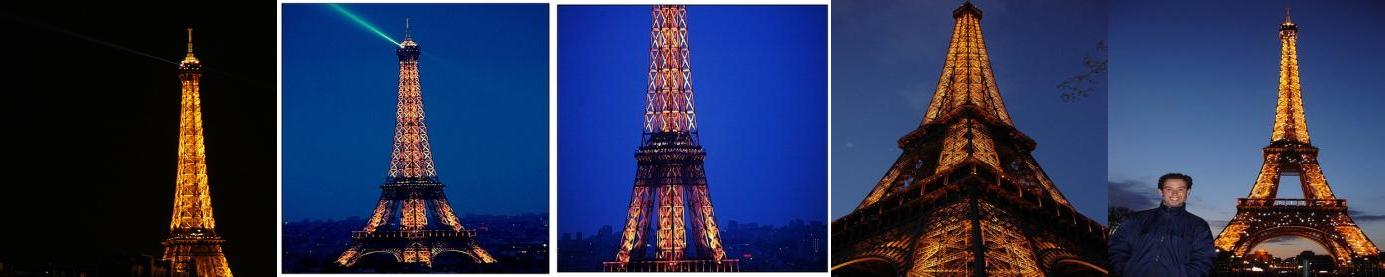}\\
\includegraphics[width=0.45\linewidth]{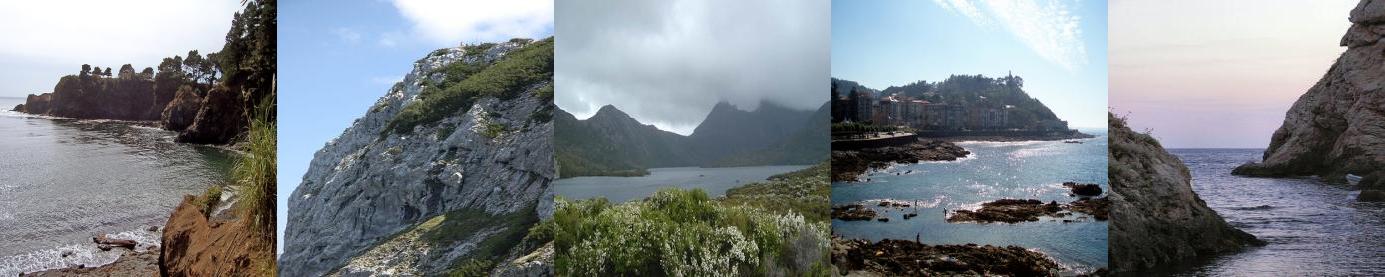}&
\includegraphics[width=0.45\linewidth]{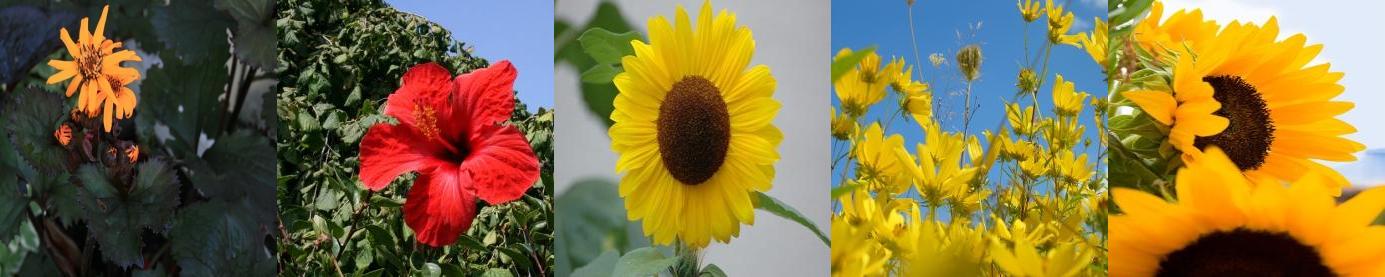}\\
\multicolumn{2}{c}{Distribution-clusters from a subset of Flickr11k~\cite{yang2008contextseer}}\\
\includegraphics[width=0.45\linewidth]{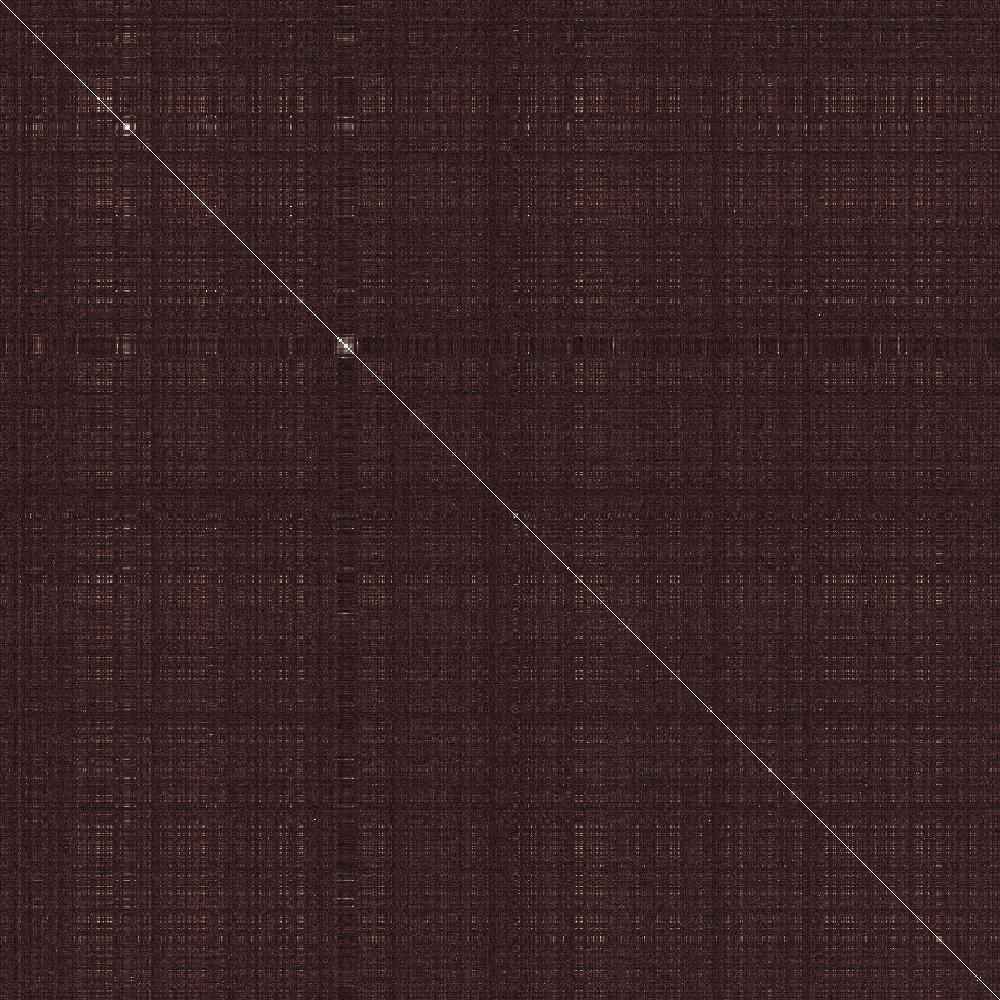}&
\includegraphics[width=0.45\linewidth]{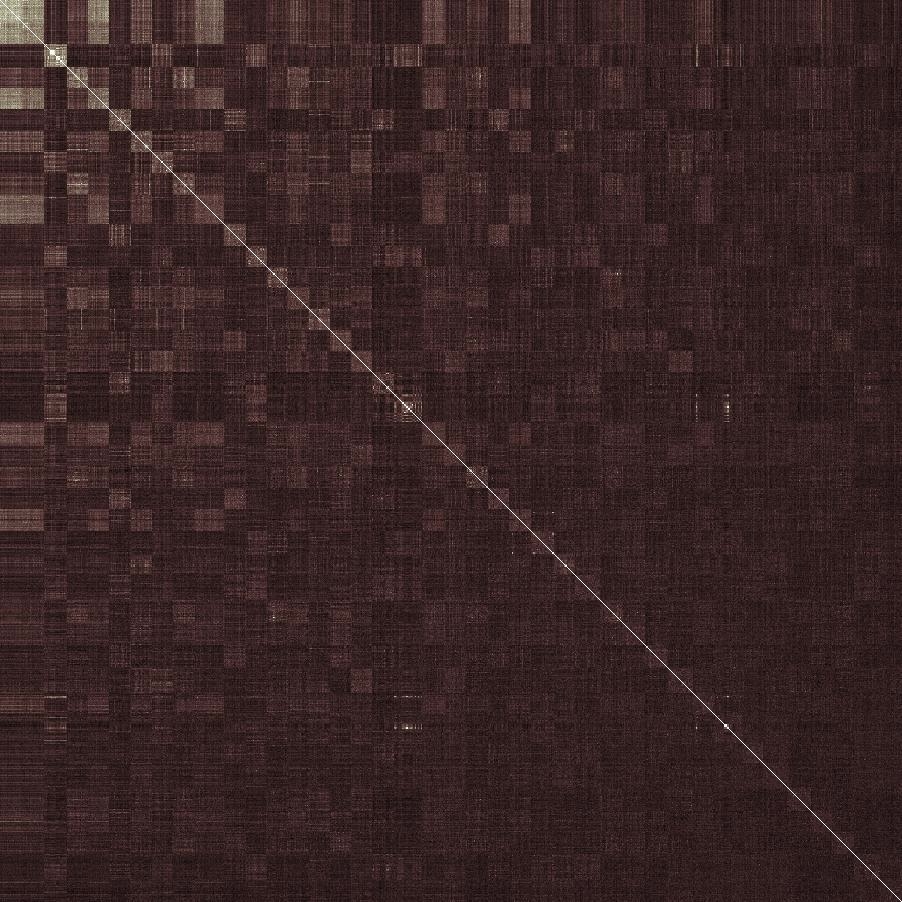}\\
\multicolumn{2}{c}{Affinity matrix before and after clustering}
\end{tabular}
\caption{Distribution-clustering on a set of random images. Our technique  captures even  highly variable distributions like flowers and scenery.
The    post-clustering  affinity matrix displays  distinctive blocky patterns predicted by our theory. 
   \label{fig:qual}}
\end{figure}

Conventional wisdom suggests  poor clustering performance is due to two intrinsic factors.
 a)
 Images are the product of a complex interplay of geometric, illumination and occlusion factors.
These are  seldom constant, causing   even images of the same location  to  vary significantly from each other.
 The extreme variability makes clustering difficult. This can be understood mathematically as each image being an instance 
of some distribution. The variability causes image data-sets to be  chaotic. Here, chaotic
is defined as having distributions whose mean separation is significantly smaller than their standard deviation.  
Clustering  chaotic data is ill-posed because data points
of different distributions mingle. This can be alleviated by enhancing invariance with
 higher dimensional image descriptors. However, it   leads to a second problem.
b) As dimensions increase,   ``contrast-loss''~\cite{aggarwal2001surprising, beyer1999nearest,domingos2012few}  occurs.
Distances between points tend to a constant, with  traditional clustering metrics becoming  ill-defined~\cite{aggarwal2001surprising,aggarwal2001outlier}.
This is considered part of the curse of dimensionality~\cite{domingos2012few, wiki}.

We offer a  different perspective in which  ``contrast-loss'' is not a problem but the  solution to clustering chaotic data. The core idea is simple.
What was previously interpreted as ``contrast-loss'' is actually  the law of large numbers causing  instances of a distribution  to concentrate on a thin ``hyper-shell''.   
The hollow shells mean  
 data points from apparently overlapping distributions  do not  actually mingle, making  choatic data  intrinsically separable.  
We encapsulate this constraint into  a second order cost that treats the rows of an affinity matrix as identifiers for  instances of  the same   distribution. We term this distribution-clustering.  

Distribution-clustering is  fundamentally different   from traditional clustering
as it can disambiguate chaotic data,  self-determine the number of clusters   and is intrinsically  robust to ``outliers'' that form their own clusters. This mind-bending result    provides an elegant solution to a problem previously deemed intractable. 
Thus, we feel it fair to conclude that 
``contrast-loss'' is a blessing rather than a curse.

%
%
%

\subsection{Related Works}
To date, there are a wide variety of   clustering algorithms~\cite{ng2002spectral,kanungo2002efficient,ester1996density,dempster1977maximum,goemans1995improved,shi2000normalized} customized to various tasks.  A  comprehensive survey is  provided in~\cite{berkhin2006survey,xu2015comprehensive}.  Despite the  variety, we believe distribution-clustering  is the first  to    utilize the peculiarities of high dimensional space. 
 The result in a   fundamentally different  clustering algorithm.

This  work is also part of  on-going research in the properties of high dimensional space. 
Pioneering research  began with  Beyer~\etal's~\cite{beyer1999nearest} discovery of  ``contrast-loss''. This   was  interpreted as 
an intrinsic  hindrance to  clustering and machine learning~\cite{aggarwal2001surprising, beyer1999nearest,domingos2012few}, motivating the development of sub-space clustering~\cite{parsons2004subspace,elhamifar2013sparse},  projective-clustering~\cite{agarwal2004k,kriegel2009clustering}, and other techniques ~\cite{Yu11,kriegel2008angle,francois2007concentration} for 
  alleviating ``contrast-loss''. 
This simplistic view has begun to change, with  recent papers observing that  ``contrast-loss'' can be 
 beneficial in detecting outliers~\cite{radovanovic2015reverse,zimek2012survey}, cluster centroids~\cite{tomasev2014role} and scoring clusters~\cite{tomasev2014role}. These results indicate a gap in our knowledge but a high-level synthesis is still lacking.

While we do not agree with Aggarwal~\etal's~\cite{aggarwal2001surprising} interpretation of  ``contrast-loss'', we are inspired by their attempts to develop a general intuition about the behavior of algorithms in high-dimensional space. This motivates us to analyze the problem  from both
intuitive and mathematical perspectives. We  hope it  contributes to the general understanding of  high dimensional space.

Within the larger context of artificial intelligence research, our work can be 
considered research on   similarity functions  surveyed by Cha~\cite{cha2007comprehensive}. Unlike
most other similarity functions, ours is statistical in nature, relying on  extreme improbability
of events to achieve separability. Such statistical similarity has been used both explicitly~\cite{bian2017gms} and 
implicitly~\cite{lin2016repmatch,delvinioti2014image, qin2011hello,zhang2012query,shi2015early} in matching and retrieval. 
Many of these problems 
may be reformulatable in terms of 
the law of large numbers, ``contrast-loss'' and  high-dimensional features. This is a fascinating and as yet unaddressed question.

Finally, distribution-clustering builds on  decades of research on image descriptors~\cite{oliva2001modeling,jegou2010,arandjelovic2016netvlad} and normalization~\cite{jegou2012negative,arandjelovic2012three}.
These works reduce variation, making the law of large numbers more impactful at lower dimensions. As distribution-clustering is based
on the law of large numbers, its performance is correspondingly enhanced.     

\section{Visualizing High Dimensions}

Our intuition about space was formed in two and three dimensions and is often misleading
in high dimensions. In fact, it can be argued that the ``contrast-loss'' curse ultimately derives from  misleading visualization.
This section aims to correct that.

At low dimensions, our intuition is  that solids with similar parameters  have significant volumetric overlap.
This is not true in high dimensions.

Consider  two high dimensional hyper-spheres which  are identical except for a small difference in radius. Their volume ratio is
\begin{equation}\label{eq:vol}  \left(\frac{r-\Delta r}{r}\right)^k=\left(1-\frac{\Delta r}{r}\right)^k \to 0, \quad k \to \infty \end{equation}
which  tends to zero as  the number of dimensions, $k \to \infty$.
This implies almost all  of a sphere's volume is concentrated at its surface.
Thus, small changes in either radius or centroid cause apparently overlapping spheres to
have near zero  intersecting volume as illustrated in \fref{fig:stat}, i.e., they become volumetrically separable! A more rigorous proof can be found in~\cite{hopcroft2014foundations}.

\begin{figure}[htp]
\centering
\includegraphics[width=1\linewidth]{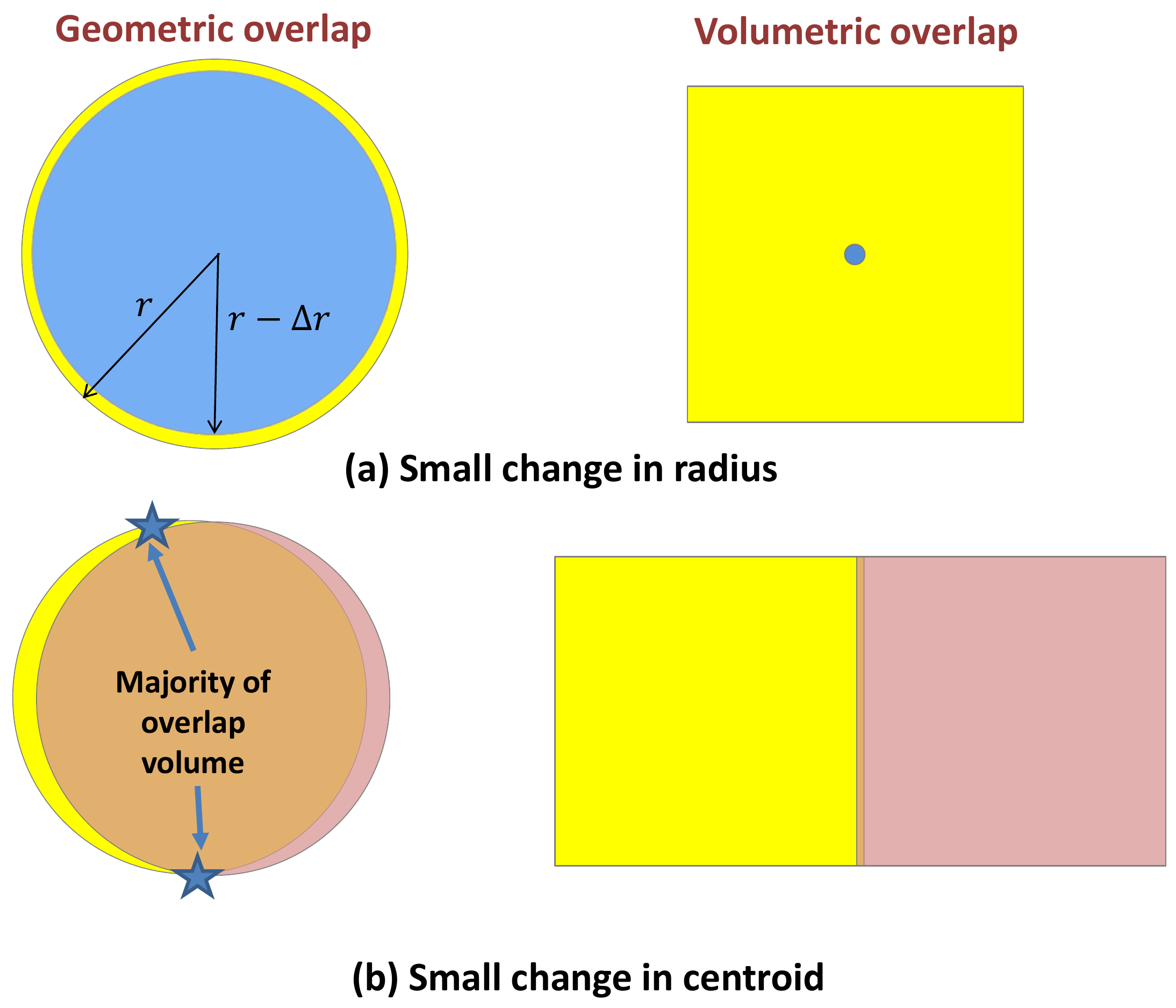}\\
\caption{Almost all of a high dimensional sphere's volume is near its surface. Thus,
small changes of radius or centroid result in almost no volumetric overlap. The apparently overlapping spheres are volumetrically separable!  \label{fig:stat}}
\end{figure}

Intriguingly, instances of a distribution behave similarly to  a hyper-sphere's volume.
\Sref{sec:dc} shows that when distributions have  many independent dimensions, their instances  concentrate on thin ``hyper-shells''.
Thus,  instances of apparently 
   overlapping distributions almost never mingle. This  makes    clustering   chaotic data by distribution  a well-posed problem, as    illustrated  in \fref{fig:stat3}.

%
%

\begin{figure}[htp]
\centering
\includegraphics[width=0.95\linewidth]{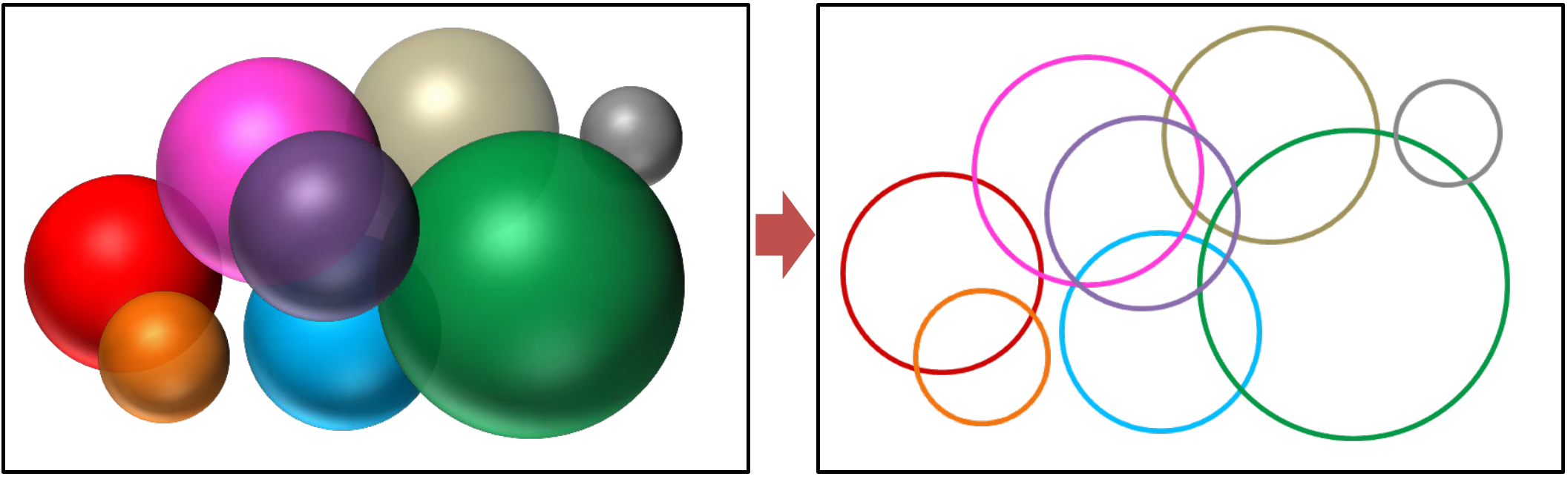}\\
\caption{\textbf{Left:} Traditional view of chaotic data as overlapping spheres.  Clustering 
  such data is deemed an ill-posed problem. 
\textbf{Right:} Our visualization. Each distribution's instances   form hollow rings. Thus, clustering  data by fitting ``hyper-shells''  becomes a well-posed
problem.    \label{fig:stat3}}
\end{figure}

\section{Distribution-Clustering (theory)}
Images are often  represented as high dimensional feature vectors, such as the $4096$-dimensional
NetVLAD~\cite{arandjelovic2016netvlad} descriptor. This section  shows how we can create indicators to group images based on their generative  distributions.

\begin{definition}
\mbox{ }
\begin{itemize}
\item $D(m, \sigma^2)$ denotes a  probability distribution with mean $m$ and variance $\sigma^2$ ;
\item $\Set_n=\{\begin{array}{cccc} 1,& 2, & \hdots, n \end{array}\}$ denotes a set of consecutive positive integers from $1$ to $n$;
\item $d(.)$ denotes a normalized squared $\Ltwo$ norm operator, \ie, for $\mathbf{x}\in \mathbb{R}^k$, $d(\mathbf{x})=\frac{\|\mathbf{x}\|^2}{k}$.
\end{itemize}

Let $\bz = \left [\begin{array}{cccc} Z_1, & Z_2, &\hdots, &Z_k\end{array} \right]^T$ denote a $k$ dimensional random vector where $Z_i$ is a random variable,
\begin{itemize}
\item  $d(.)$ operator  can also be applied on a random vectors $\mathbf{Z}$. $d(\mathbf{Z})=\frac{\|\mathbf{Z}\|^2}{k}$ is a random variable formed by averaging $\bz$'s squared elements. 
\item $E(\mathbf{Z})=\left [\begin{array}{cccc} \E(Z_1),& \E(Z_2), & \hdots, & \E(Z_k) \end{array} \right]^T$ is a vector of each dimension's expectation;
\item With a slight abuse of notation, we define $\Var(\mathbf{Z})=\sum_{i=1}^k \frac{\var(Z_i)}{k}$, as the average variance over all dimensions.
\end{itemize}
\end{definition}

\subsection{Passive Sensing Model}
\label{sec:images_ind}

Many data sources (like cameras) can be modeled as passive sensors.
Data points  (like image descriptors) $\{\bbx^{(i)}: i\in \Set_n\}$, are instances  of   random  vectors $\{\mathbf{X}^{(i)}: i\in \Set_n\}$ representing 
 environmental factors that influence  sensory outcome, \eg,  camera position, time of day, weather conditions.
As  sensing does not influence the environment, all random  vectors are  mutually independent. 
Our goal  is to cluster    $\mathbf{x}^{(i)}$ instances by their underlying  $\mathbf{X}^{(i)}$ distributions.

\subsection{Quasi-ideal Features}
\label{sec:quasi-ideal}

An ideal feature descriptor has statistically independent dimensions. However, this is hard to ensure  in practice. A more  practical  assumption is the  quasi-independence  in  condition \ref{def:ind}.

\begin{condition}
\label{def:ind}
\emph{Quasi-independent:}
A set of $k$ random variables $\{X^{(i)}: i\in \Set_k\}$ are quasi-independent, if and only if,  as $k\to \infty$, each random variable has
 finite number of pairwise dependencies. That is,
 let $\mathcal{A}$ be the set of all pairwise dependent variables
and $\mathbf{1}$ be an indicator function,
there  exists $t\in \mathbb{Z}^+$ such that
$$
\sum_{j=1}^k \mathbf{1}_{\mathcal{A}}(\{X^{(i)}, X^{(j)}\}) \leq t, \quad \forall i \in \Set_k.
$$

\end{condition}
\noindent Quasi-independence is approximately  equivalent to requiring  information increases proportionally with number of random variables.
When the random variables are concatenated into a feature, we term it quasi-ideal.

\begin{condition}
\label{def:prac}
\emph{Quasi-ideal:} A $k$-dimensional random vector $\mathbf{X}$ is quasi-ideal, if and only if, as $k\to \infty$,
the variance of all its elements are finite and the set of all its elements, $\{X_i: i\in \Set_k\}$,
is quasi-independent.
\end{condition}

Treating  the links of an infinitely long Markov chain as  feature dimensions
would create a quasi-ideal feature. This is useful in computer vision, as pixel values  have Markov like properties  of  some statistical dependence on
 neighbors but  long range statistical independence. Hence,
many image based descriptors can be modeled as quasi-ideal.

Practicality aside, quasi-ideal features have useful mathematical properties,  as 
they permit 
 the law of large numbers to apply to distance metrics. This leads to  interesting results
summarized
 in Lemmas \ref{lemma:base}
and   \ref{lemma:pair}.

\begin{lemma}
\label{lemma:base}
Let $\mathbf{X}$ be  a  quasi-ideal random vector with dimension $k\to \infty$. The normalized squared $\Ltwo$ norm of any instance is almost surely  a constant:
\begin{align}
d(\mathbf{X})=\frac{\|\mathbf{X}\|^2}{k}=\frac{\sum_{i=1}^k X_i^2}{k}\sim D\left(m, \sigma^2\right), \sigma \to 0,
\end{align}
\end{lemma}
\begin{proof}

As $\mathbf{X}$ is quasi-ideal, the set of squared elements $\{X_1^2, X_2^2, \hdots, X_k^2\}$ form a covariance matrix where the sum of elements in any row is
bounded by some positive real number $t$, \ie,  
\begin{align}
\sum_{j=1}^k \cov(X_i^2, X_j^2) \leq t,  \quad\forall i \in \Set_k.
\end{align}
This implies
\small
\begin{align}
\Var(d(\mathbf{X}))=\frac{\sum_{i=1}^k\sum_{j=1}^k \cov(X_i^2, X_j^2)}{k^2} < \frac{kt}{k^2}<\frac{t}{k}.
\end{align}
\normalsize
Thus, as $k\to \infty$, variance tends to zero.
\end{proof}

\begin{lemma}
\label{lemma:pair}
Let $\mathbf{X}$ and $\mathbf{Y}$ be statistically independent, quasi-ideal random vectors. As the dimensions $k\to \infty$.
\begin{align}
d(\mathbf{X}-\mathbf{Y})=\frac{\|\mathbf{X}-\mathbf{Y}\|^2}{k}\sim D\left(m, \sigma^2\right), \sigma \to 0,
\end{align}
where
\hbox{$m= \Var(\mathbf{X})+ \Var(\mathbf{Y})+ d(\E(\mathbf{X})-\E(\mathbf{Y})).$}
\end{lemma}

\begin{proof}
As $\mathbf{X}$ and $\mathbf{Y}$ are quasi-ideal,  random vector $\mathbf{X}-\mathbf{Y}$ is also quasi-ideal. Using Lemma~\ref{lemma:base}, we know $d(\mathbf{X}-\mathbf{Y})$'s variance tends to zero. The expression for its mean is:
\small
\begin{align}
m=&\E\left(\frac{\|\mathbf{X}-\mathbf{Y}\|^2}{k}\right) \nonumber \\
=&\sum_{i=1}^k \frac{\E\left(X_i^2\right)}{k}+\frac{\E\left(Y_i^2\right)}{k}-\frac{2\E\left(X_iY_i\right)}{k} \nonumber\\
=&\sum_{i=1}^k \frac{ \var\left(X_i\right)+\var\left(Y_i\right) }{k} \nonumber\\
&+ \sum_{i=1}^k  \frac{\E\left(X_i\right)^2-2\E\left(X_i\right) \E\left(Y_i\right)+ \E\left(Y_i\right)^2}{k} \nonumber\\
=& \Var(\mathbf{X})+\Var(\mathbf{Y}) +\sum_{i=1}^k\frac{ \left(\E\left(X_i\right)-\E\left(Y_i\right)\right)^2}{k} \nonumber\\
=& \Var(\mathbf{X})+\Var(\mathbf{Y}) +d(\E(\mathbf{X})-\E(\mathbf{Y})). \nonumber
\end{align}
\end{proof}\normalsize
Lemma \ref{lemma:pair} is similar in spirit to Beyer \etal's~\cite{beyer1999nearest}   ``contrast-loss'' proof.
However, it accommodates realizations from different distributions,  introduces a more practical quasi-independence assumption
and is simpler to derive.

Unlike~\cite{aggarwal2001surprising},  we consider ``contrast-loss'' an opportunity not a liability.  Lemma \ref{lemma:pair} proves that distance between instances almost always
depend only on the mean and variances of the underlying distributions and not on  instances' values. This makes  distance between instances a potential   proxy  for identifying their underlying distributions.

\subsection{Distribution-clusters }
\label{sec:dc}

Identifying  data points from ``similar'' distributions requires a definition of ``similarity''.  Ideally, we would follow Lemma \ref{lemma:pair}'s intuition and define ``similarity''   as having the same  mean and average variance. However,  the definition needs  to accommodate  dimensions tending to infinity. This leads to a  distribution-cluster based ``similarity'' definition.

Let $\Omega_\bx =\{\mathbf{X}^{(i)}: i\in \Set_n\} $ be a set of independent $k$-dimensional
random vectors.
$k \to \infty$ such that the random vectors satisfy quasi-ideal conditions.

\begin{condition}
\label{cond:dc}
\emph{Distribution-cluster:} $\Omega_\bx$ forms a distribution-cluster if and only if:
\begin{itemize}
\item The normalized squared $\Ltwo$ norm distance between any two distribution mean is zero, \ie,
\begin{equation}
\label{eq:equi_mean}
d(\E(\mathbf{X}^{(i)})-\E(\mathbf{X}^{(j)}))=0, \quad \forall i,j \in \Set_n; \\
\end{equation}
\item All  distributions have  the same average variance, \ie,
\begin{equation}
\label{eq:equi_var}
 \Var(\mathbf{X}^{(i)})= \Var(\mathbf{X}^{(j)}), \quad \forall i,j \in \Set_n.
\end{equation}
\end{itemize}
\end{condition}

%
%

As dimensions tend to infinity,  instances of a distribution-cluster  concentrate on a ``thin-shell''. This is proved in
theorem \ref{theorem:diff} and
validates  \fref{fig:stat3}'s intuition. The ``hollow-center'' means  data points from apparently overlapping distributions almost never  mingle,  creating the
potential for
  clustering  chaotic data.

\begin{theorem}
\label{theorem:diff}
If $\Omega_\bx$ is a distribution-cluster with average variance $v$,
the normalized squared distance of its instances from the cluster centroid will almost  surely
be  $v$, \ie, $\Omega_\bx$'s instances   form a thin annulus  about it's centroid.
\end{theorem}
\begin{proof}
 Without loss of generality, let $\mathbf{m}=E(\bx^{(1)})$.
Let $\bx$ in Lemma \ref{lemma:pair} be $\bx^{(i)}$ and  $\by$ in Lemma \ref{lemma:pair} be  a  distribution with mean $\mathbf{m}$ and variance $0$.
This gives an expression:
\begin{equation}
p\left(d(\bx^{(i)}-\mathbf{m})= v \right) \to 1 \quad \forall i \in \Set_n
\end{equation}

\end{proof}

\subsection{Grouping  Data by Distribution}
\label{sec:id_c}

We seek to group a set of data points by their underlying  distribution-clusters. This
is achieved by proving that data points of a distribution-cluster share unique identifiers that we term cluster-indicators.

\begin{theorem}
\label{theorem:final}
\emph{Cluster-indicator:}
Let $\by$ be a quasi-ideal random vector that is  independent of all  $\Omega_\bx$'s random vectors.

 $\Omega_\bx$ forms  a distribution-cluster  (\emph{c.f}~condition \ref{cond:dc}), if and only if,
for any valid  random vector  $\by$,
there exists a real number $b_\mathbf{Y}$
such that
\begin{equation}
\label{eq:y_proof}
\begin{split}
p(d(\by-\bx^{(i)}) =b_\mathbf{Y}) \to 1, \quad \forall i\in \Set_n.
\end{split}
\end{equation}

%

\end{theorem}

\begin{proof}
%
%


First,  the \emph{if} part is proved. Given $\Omega_\mathbf{X}$ is a  distribution-cluster,
 \eref{eq:y_proof} is a direct result of Lemma \ref{lemma:pair}, where the distance between
instances of quasi-ideal distributions are almost surely determined by the distributions' mean and average variances.

Moving on to the \emph{only if} proof, where it is given that  $\Omega_\bx$ satisfies \eref{eq:y_proof}'s cluster-indicator.
 \textit{W.l.o.g}, we consider only elements $\mathbf{X}^{(1)}, \mathbf{X}^{(2)}$.  Let $\mathbf{Y}$ be independent but identically
distributed with    $\mathbf{X}^{(1)}$.

From Lemma \ref{lemma:pair}, we know that
\begin{itemize}
\item $p(d(\mathbf{Y}-\mathbf{X}^{(1)})=m_1 )\to 1,$ where

\noindent   $m_1=2\Var(\mathbf{X}^{(1)});$
\item

$p(d(\mathbf{Y}-\mathbf{X}^{(2)})= m_2)\to 1,$ where
$$m_2=\Var(\mathbf{X}^{(1)})+\Var(\mathbf{X}^{(2)}) +d(\E(\mathbf{X}^{(1)})-\E(\mathbf{X}^{(2)})).$$

\end{itemize}

 \Eref{eq:y_proof} means $b_{\mathbf{Y}}=m_1= m_2$, implying:\small
 \begin{equation}
\label{eq:one_diff}
\begin{split}
&2\Var(\mathbf{X}^{(1)})\\
=&\Var(\mathbf{X}^{(1)})+\Var(\mathbf{X}^{(2)}) +d(\E(\mathbf{X}^{(1)})-\E(\mathbf{X}^{(2)})).
\end{split}
\end{equation}
\normalsize

Similarly,  treating $\by$ as
independent but identically
distributed with  $\mathbf{X}^{(2)}$ implies
\small
 \begin{equation}
\label{eq:two_diff}
\begin{split}
&2\Var(\mathbf{X}^{(2)})\\
=&\Var(\mathbf{X}^{(1)})+\Var(\mathbf{X}^{(2)}) +d(\E(\mathbf{X}^{(1)})-\E(\mathbf{X}^{(2)})).
\end{split}
\end{equation}
\normalsize

Solving \eqref{eq:one_diff} and  \eqref{eq:two_diff} yields
$$\Var(\mathbf{X}^{(1)})=\Var(\mathbf{X}^{(2)}),  \quad d(\E(\mathbf{X}^{(1)})-\E(\mathbf{X}^{(2)}))=0.$$
This proves that $\mathbf{X}^{(1)}, \mathbf{X}^{(2)}$ are members of a distribution-cluster, (\emph{c.f}~condition \ref{cond:dc}).
Repeating the process with all element pairs of $\Omega_\bx$ will show they belong
to  one  distribution-cluster. This completes the \emph{only if} proof.

\end{proof}

As argued in \sref{sec:images_ind}, image descriptors  can be modeled  as instances
of independent, quasi-ideal  random vectors, \ie, a set of image descriptors $\Omega_\mathbf{x} =\{\mathbf{x}^{(i)}: i\in \Set_n\} $ can be considered
 instances of the respective random vectors in $\Omega_{\bx}$. 
Theorem \ref{theorem:final} implies that descriptors from the
same distribution-cluster will (almost surely) be equi-distance to any other descriptor. Further, it  is a unique
property of descriptors from the same distribution-cluster. This allows
descriptors to be unambiguously assigned to  distribution-clusters. In summary,
distribution-clustering of images (and other passive sensing data)  is a well-posed problem, per the definition  in McGraw-Hill dictionary of scientific and technical terms~\cite{parker1984mcgraw}:

\begin{itemize}
\item  A solution exist.  This follows from  Theorem \ref{theorem:final}'s  \emph{if} condition where
cluster-indicators  almost surely (in practice it can be understood as surely) identify   all  data points
of a distribution-cluster;

\item A solution is unique.
This follows from  Theorem \ref{theorem:final}'s  \emph{only  if} condition which
means  cluster-indicators   almost never confuse data points of different distributions.
This can also be understood as proving  intrinsic  separability of instances from different distribution-clusters;

\item
The solution's behavior changes continuously with the initial conditions.
The $b_{\mathbf{Y}}$ cluster-indicator in \eref{eq:y_proof}   vary continuously
with  the mean and
average variance of the underlying  distribution-cluster. This follows from
Lemma \ref{lemma:pair}'s
 expression for $b_{\mathbf{Y}}$.

\end{itemize}

\section{Distribution-Clustering (practical)}
Our goal is to use theorem \ref{theorem:final}'s cluster-indicators  to group data points  by  their underlying distributions.
From theorem \ref{theorem:final}, we know that if $\mathbf{x}^{(i)},\mathbf{x}^{(j)}$ are instances of the same distribution-cluster,
the affinity matrix's  $i,j$ rows/ columns    will be near identical. To exploit this, we  define second order features  as columns of the
affinity matrix.
Clustering is achieved by grouping  second-order features.

\subsection{Second-order Affinity}
Let  $ \Omega_{\mathbf{x}} =\{\mathbf{x}^{(i)}: i\in {S}_n\}$ be a set of realizations, with an associated affinity matrix $\mathbf{A}_{n\times n}$:
\begin{equation}
\label{eq:A}
\mathbf{A}(i,j)=d(\mathbf{x}^{(i)}-\mathbf{x}^{(j)}).
\end{equation}

The  columns of  $\mathbf{A}$  are denoted as $\mathbf{a}^{(i)}=\mathbf{A}(:,i)$. Treating  columns as features yields   a set of second-order features
$\{\mathbf{a}^{(i)}: i\in \Set_n\}$. The elements of $\mathbf{a}^{(i)}$ encodes the distance between vector $\mathbf{x}^{(i)}$ and all others in $\Omega_{x}$.

From theorem \ref{theorem:final}, we know that if and only if the distributions underlying
 $\mathbf{a}^{(i)}, \mathbf{a}^{(j)}$ come  from the same distribution-cluster, all their elements, except the $i^{th}$ and $j^{th}$ entries,
are almost surely  identical.  This is encapsulated as a second-order distance:
\begin{equation}
\label{eq:clus_d}
d' (\mathbf{a}^{(i)},\mathbf{a}^{(j)})=\sum_{k\in \Set_n\setminus\{i,j\}}\left( a^{(i)}_k-a^{(j)}_k\right)^2.
\end{equation}
which should be zero if $i,j$ belong to the same distribution-cluster. 
The presence of clusters of identical rows causes the   post-clustering  affinity matrix to display a distinctive blocky pattern shown in \fref{fig:qual}.

Second order distance can be embedded in existing clustering algorithms.
For techniques like spectral-clustering~\cite{zelnik2005self} which require an affinity matrix,
a second-order affinity matrix is  defined as  $\mathbf{A}'_{n\times n}$: $$\mathbf{A}'(i,j)=d' \left(\mathbf{a}^{(i)},\mathbf{a}^{(j)}\right ).$$
If  $n$ is large, $d' \left(\mathbf{a}^{(i)},\mathbf{a}^{(j)}\right )\approx \|\mathbf{a}^{(i)}-\mathbf{a}^{(j)}\|^2$.
This allows   second order $\{\mathbf{a}^{(i)}\}$  features to be used directly in clustering algorithms like k-means, which require
feature inputs.
%

Incorporating second-order constraints into a prior clustering algorithm does not fully utilize
theorem \ref{theorem:final}. 
This is because realizations of the same distribution-cluster   have zero second-order-distance, while  most clustering
algorithms only apply a distance penalty. This motivates  an alternative  solution we term
distribution-clustering

\subsection{Implementing  Distribution-clustering}
\label{sec:alt_clus}
%

Distribution-clustering 
 can be understood as  identifying indices whose mutual second-order distance is near zero.
These are grouped into one  cluster and the process repeated to identify more clusters.  

An algorithmic overview  is as follows.  
 Let $i,j$ be the indices of its smallest off-diagonal entry of affinity matrix $\mathbf{A}$. If $\{\mathbf{x}^{(i)}, \mathbf{x}^{(j)}\}$ are instances of  a distribution-cluster,  Lemma \ref{lemma:pair}
states the average cluster variance is
$\mathbf{A}(i,j)/2$. Thus they are the data-set's lowest average variance  distribution-cluster. Initialize $\{\mathbf{x}^{(i)}, \mathbf{x}^{(j)}\}$
as a candidate distribution-cluster. New members are recruited by finding vectors whose average second-order distance from all  distribution-cluster candidates is
 less than  threshold $\tau$. If a candidate distribution-cluster grows to have no less than   $m$ members,  accept it.
Irrespective of the outcome, remove $\{\mathbf{x}^{(i)}, \mathbf{x}^{(j)}\}$ from  consideration as candidate clusters.
Repeat on  un-clustered data  till all data points are  clustered or it is impossible to form a candidate cluster.  Some data may not  be accepted in any
cluster and remain outliers. Details are  in Algorithm \ref{algo:dis_clus}.
For whitened descriptors~\cite{jegou2012negative,arandjelovic2012three}, typical parameters are $\tau=0.07,m=5$. 

\begin{algorithm}[]
\textbf{Input:} Affinity matrix $\mathbf{A}_{n\times n}$\\
\textbf{Output:} Vector of cluster labels $\mathbf{L}$ \\
\textbf{Initialization:} 1) Set of un-assigned image indices: $\mathcal{S}_n =\{1, 2,\hdots, n\}$; \hskip 0.2cm 2)
Set of non-diagonal elements of $\mathbf{A}$: $\mathcal{E}=\{\mathbf{A}(i,j)\}, i \neq j$; \hskip 0.2cm 3) Initialize $\mathbf{L} =\mathbf{0}_{n\times1}$;
4) Set label counter $c=1$\;
 \While {$\Set_n \neq \emptyset$}{
  	Find $i$ and $j$ corresponding to $\min(\mathcal{E})$\;
    Create candidate cluster  $\mathcal{H} =\{i,j\}$\;
	 \For{$s\in \Set_n$}{
		\If{$\frac{1}{|\mathcal{H}|}\sum_{h\in \mathcal{H}} d'(\mathbf{a}^{(s)}, \mathbf{a}^{(h)})< \tau$,}{
			insert $s$ into $\mathcal{H}$
		}
	}
	\eIf{$|\mathcal{H}|<$ m}{
 		delete  $\mathbf{A}(i,j)$ from $\mathcal{E}$  \;
 	}{
 		 accept  cluster $\mathcal{H}$ and assign its elements  a unique label\;
		$\mathbf{L}(h)=c, \;\;\;\; \forall h \in \mathcal{H}$\;
		$c:=c+1$\;

		\For{$h \in \mathcal{H}$ } {
			delete $\mathbf{A}(:,h)$ and $\mathbf{A}(h,:)$ from set $\mathcal{E}$\;
			delete $h$ from set $\Set_n$.
		}
	}
 }
 \caption{Distribution-clustering \label{algo:dis_clus}}
\end{algorithm}
Relative to other clustering techniques, distribution-clustering  has
many theoretical and practical advantages:
\begin{itemize}
\item Clustering  chaotic data is a well-posed problem  (\cf \sref{sec:id_c});
\item No pre-definition of cluster numbers is required;
\item Innate robustness to ``outliers'' which form  their own clusters.
\end{itemize}

\section{Clustering}

\paragraph{Simulation Results} use quasi-ideal features created from a mixture of 
uniform and Gaussian distributions. To evaluate 
the effect of increasing dimensionality, the number of dimensions is increased from 
$1$ to $4000$. Two sets are evaluated. The ``Easy'' set has 
wide separation  of  underlying
distributions while the ``Difficult'' set has little separation. 
Results are presented in
\fref{fig:simu_dimension}. We  compare three different distance measures on k-means clustering~\cite{Lloyd82leastsquares,arthur2007k}: $\ell_2$ norm, $\ell_1$ norm and our proposed second-order
distance in \eref{eq:clus_d}. We also compare spectral clustering~\cite{zelnik2005self} with   $\ell_2$  and  second-order distance. Finally, we provide a system to system comparison between 
our distribution-clustering,  k-means and  spectral-clustering. At low dimensions, the second-order distance gives results comparable to other algorithms.  However, performance steadily improves with number of dimensions. Notably, only algorithms
which employ second-order distance  are effective on the ``Difficult'' set. This validates the theoretical prediction that  (the previously ill-posed problem of) clustering  chaotic data is  made  well-posed 
by the second-order distance.  

To  study the effect of mean separation on clustering performance,
we repeat the previous experiment under similar conditions,   except  the number of dimensions are kept constant and the mean separation
progressively reduced to zero. Results are presented in \fref{fig:simu}. Note that second-order distance ensures
clustering performance is relatively invariant to mean  separation.

\begin{figure}[t]
\centering
\setlength\tabcolsep{1.5pt} 
\begin{tabular}{cccc}
  \includegraphics[width=0.34\linewidth]{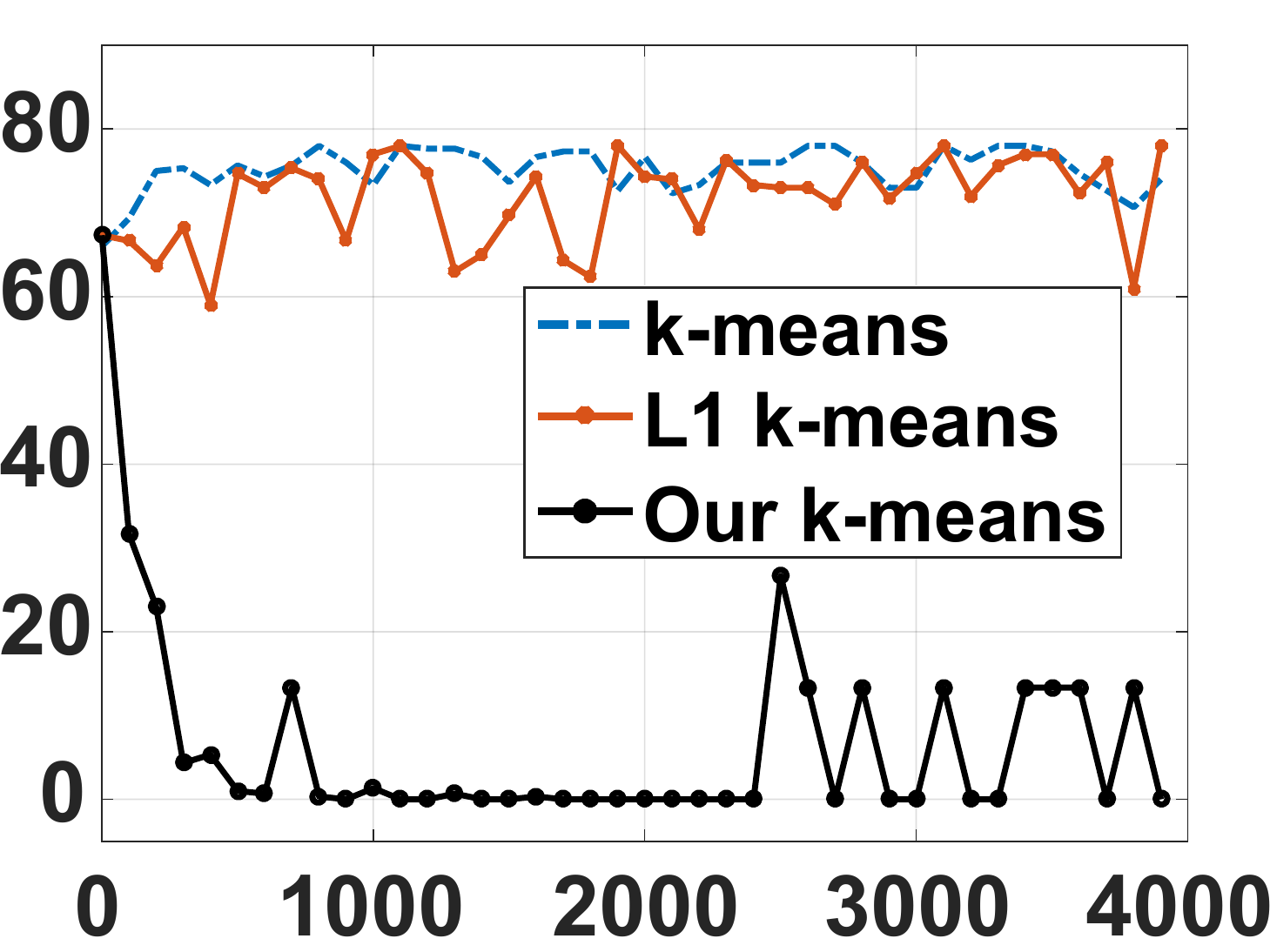}
  &\includegraphics[width=0.34\linewidth]{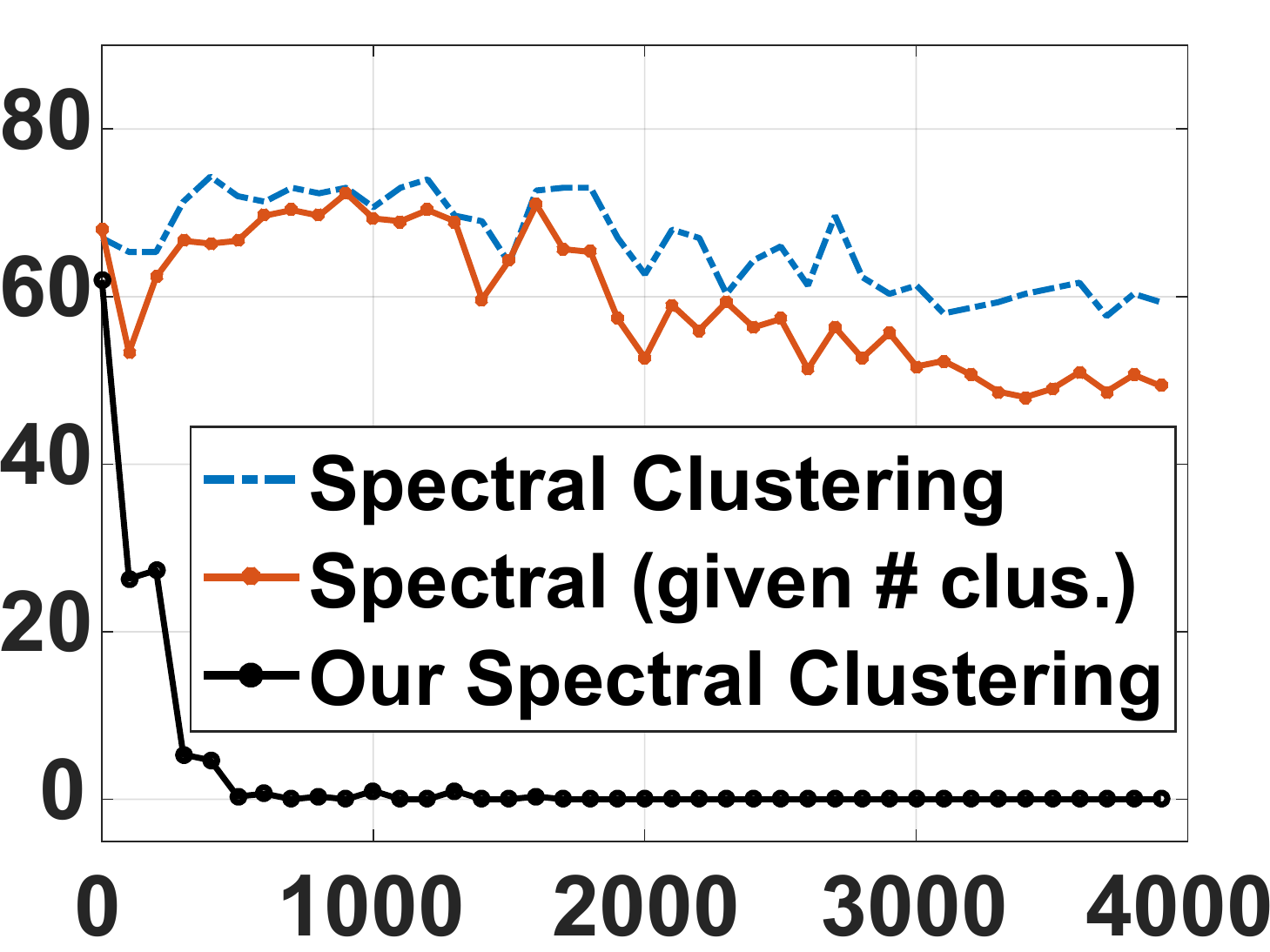}
  &\includegraphics[width=0.34\linewidth]{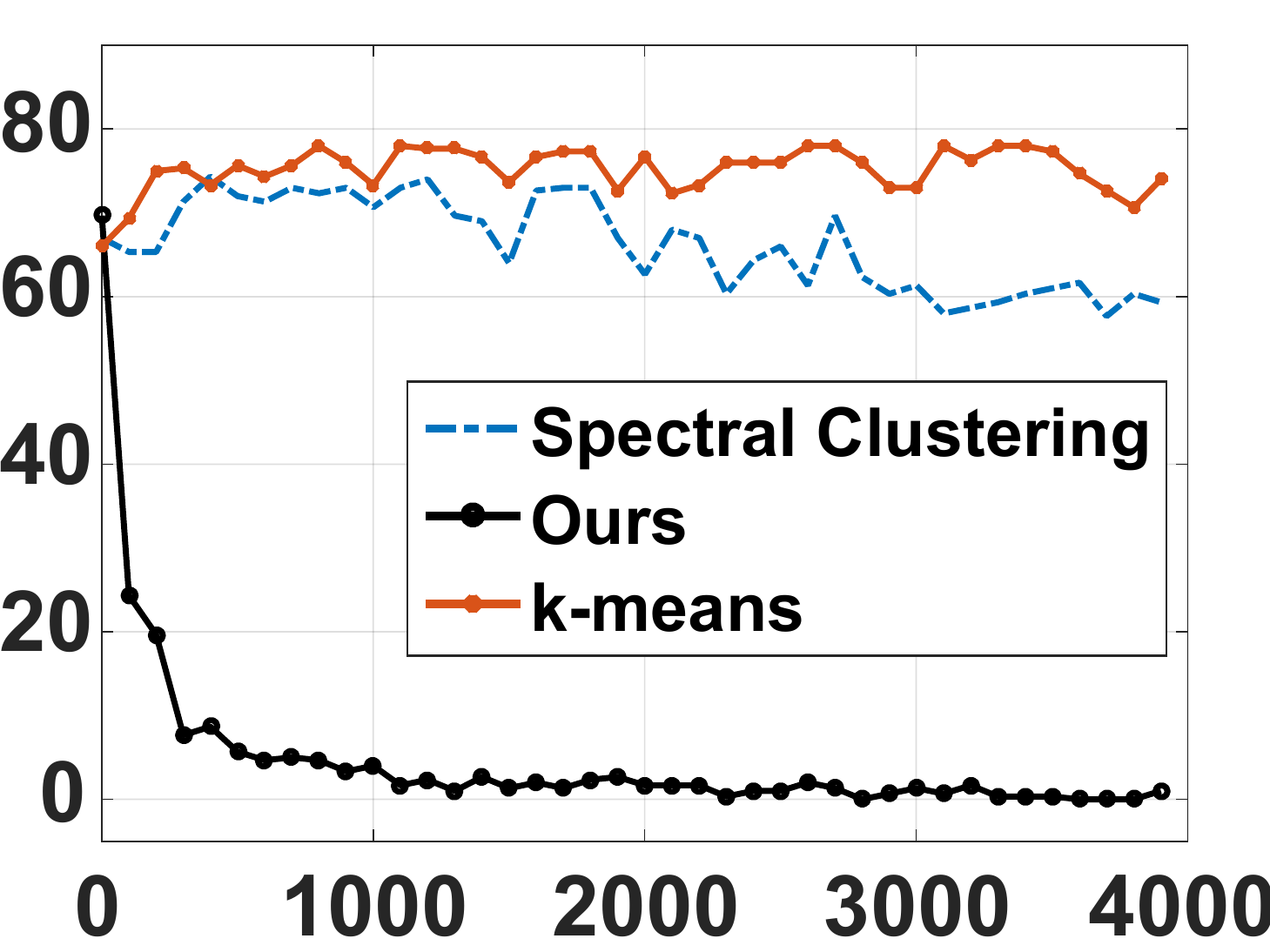}\\
  \multicolumn{3}{c}{Difficult}\\
  \includegraphics[width=0.34\linewidth]{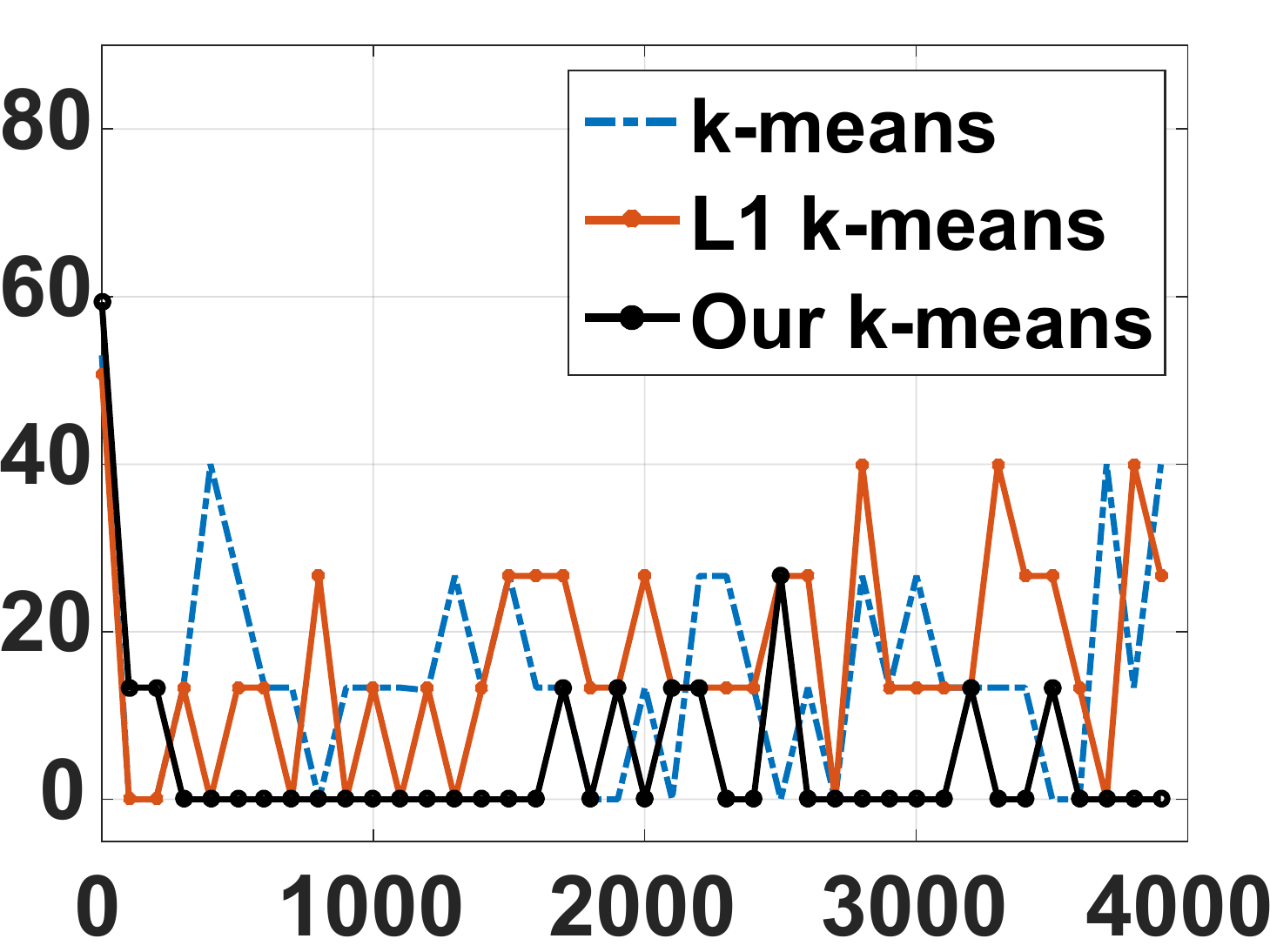}
  &\includegraphics[width=0.34\linewidth]{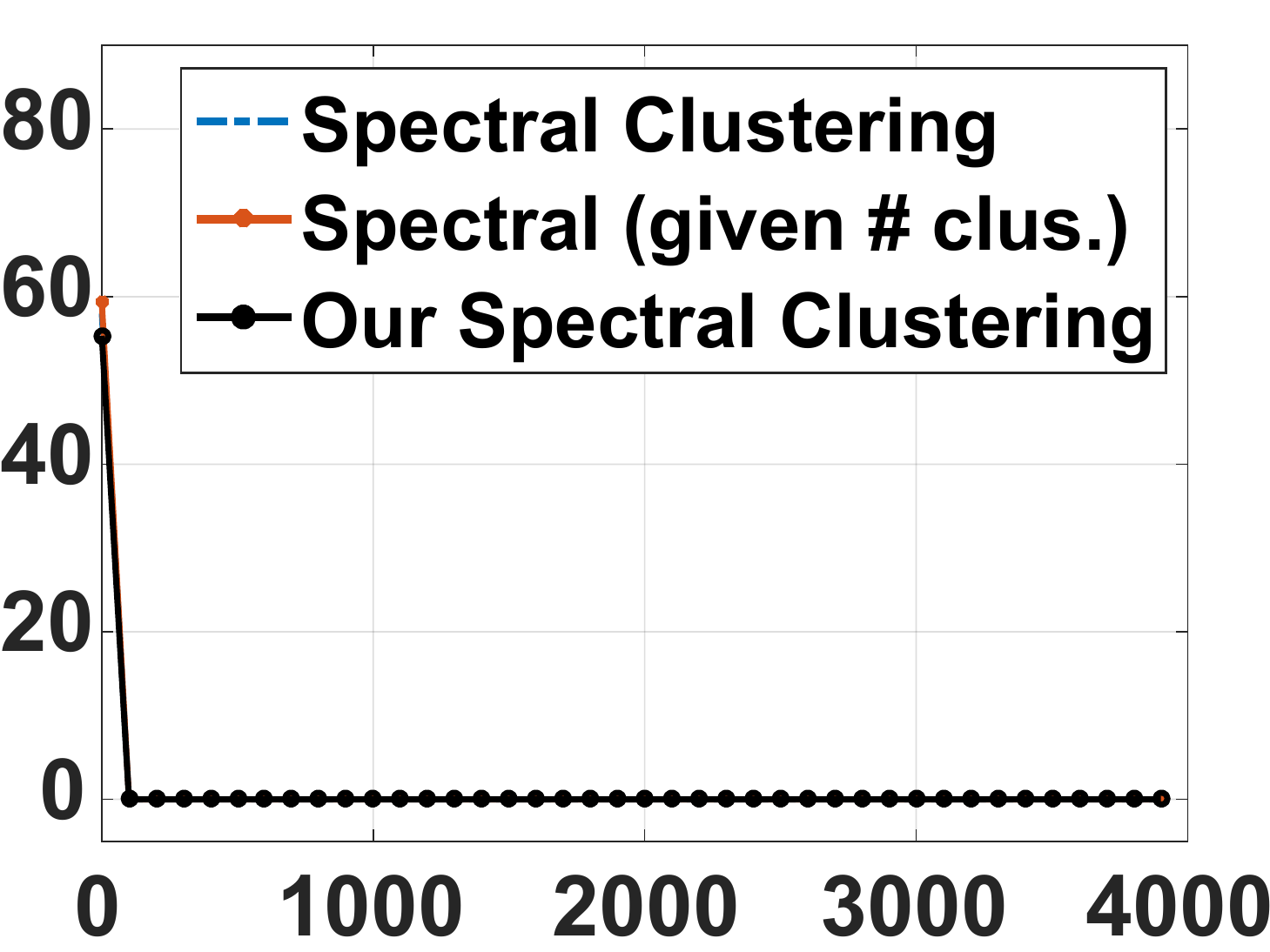}
  &\includegraphics[width=0.34\linewidth]{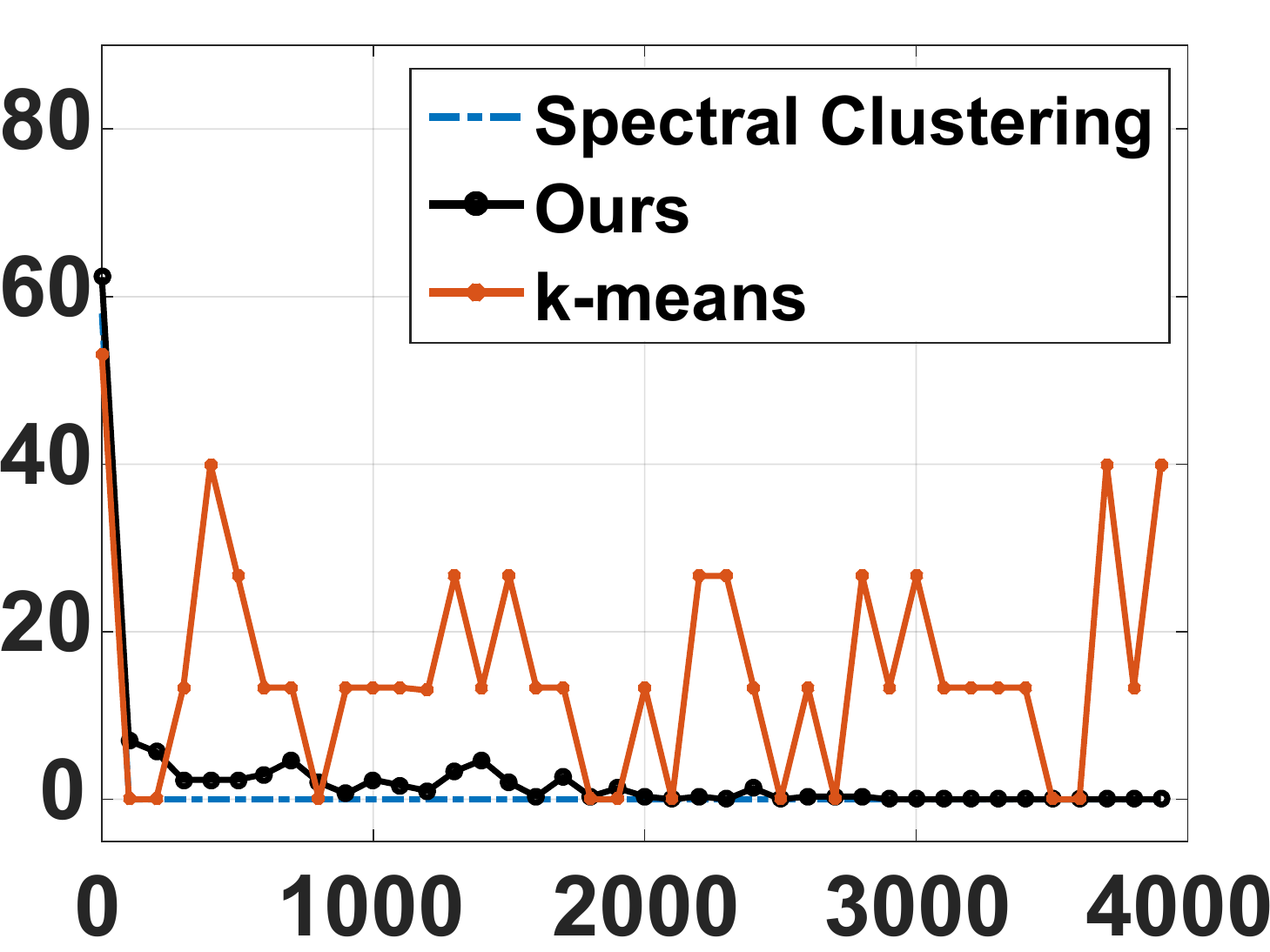}\\
  \multicolumn{3}{c}{Easy}\\
  \end{tabular}
\caption{
Simulation with increasing number of dimensions.
Left: k-means with $\ell_2$ norm, k-means with $\ell_1$ norm and our  second-order distance. Center: Spectral clustering with automatic
detection of cluster numbers~\cite{zelnik2005self}, given number of clusters~\cite{zelnik2005self} and   second-order
distance (\eref{eq:clus_d}). Right: System to system comparison of distribution-clustering,  k-means and spectral  clustering.
Only high dimensional, second-order distance algorithms 
are effective on ``Difficult'' data. 
 \label{fig:simu_dimension}}
\end{figure}

%
\begin{figure}[thp]
\centering
\setlength\tabcolsep{1.5pt} 
\begin{tabular}{ccc}
\includegraphics[width=0.32\linewidth]{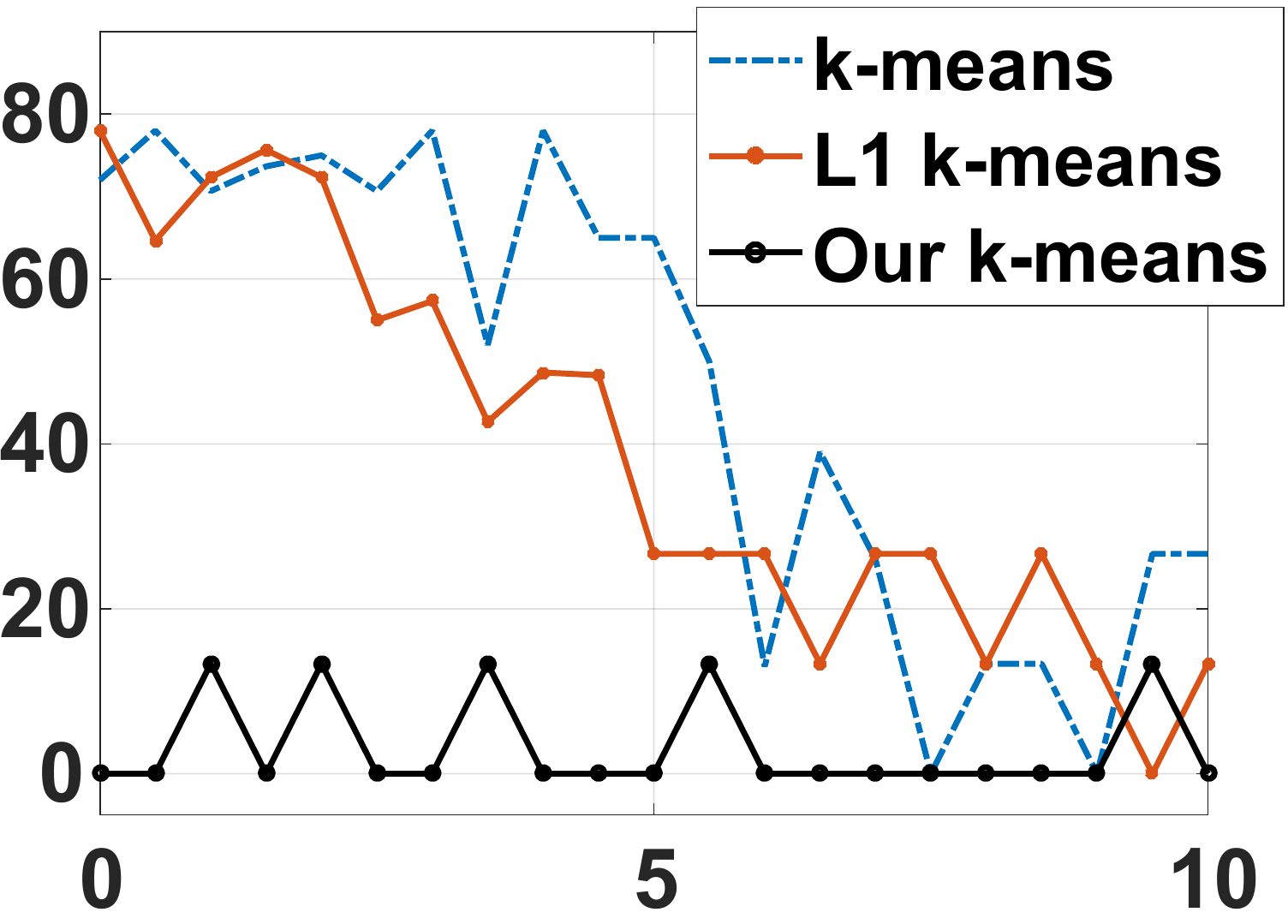}
&\includegraphics[width=0.32\linewidth]{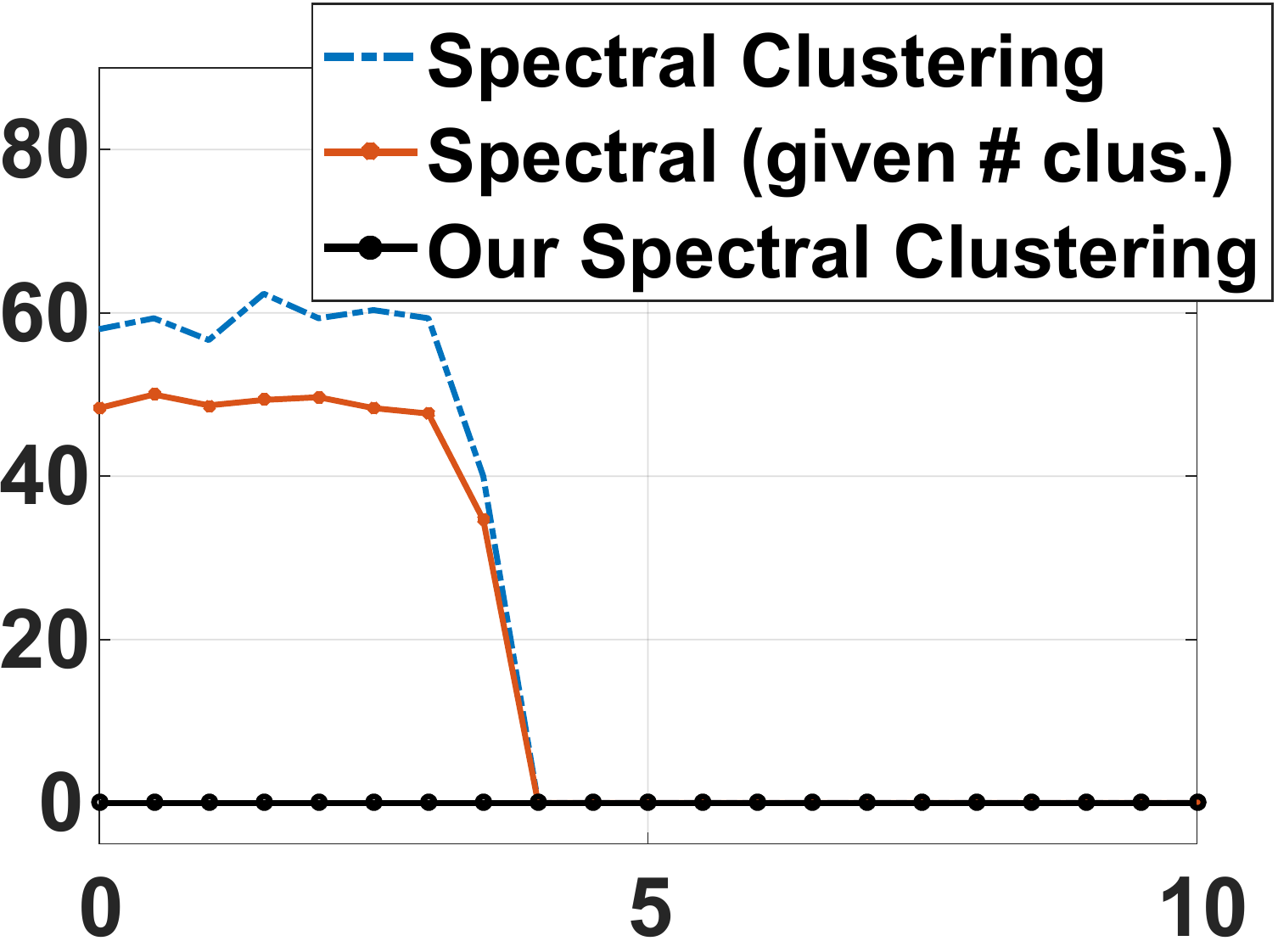}
&\includegraphics[width=0.33\linewidth]{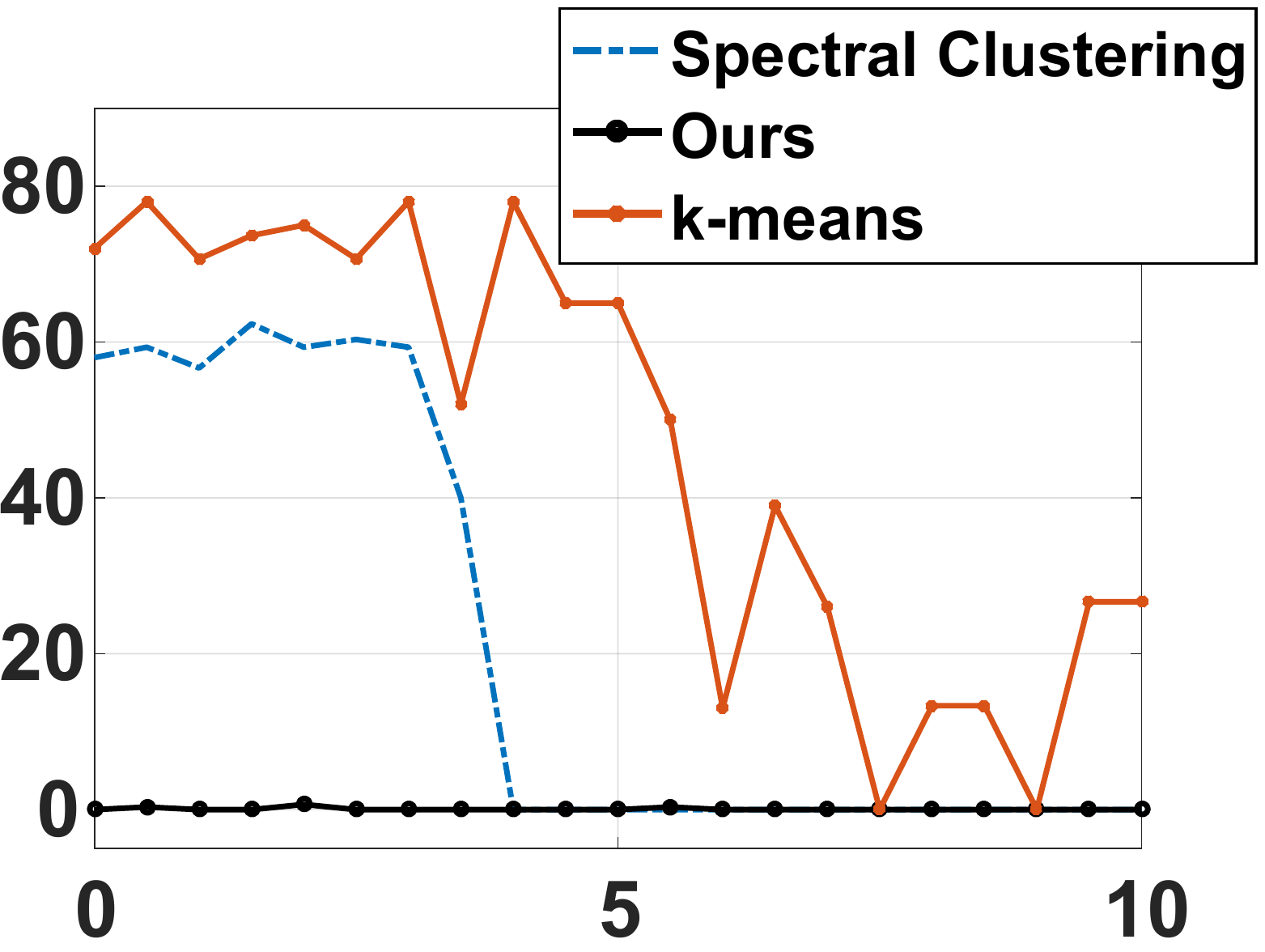}\\
\multicolumn{3}{c}{$\%$ Mis-classified vs. Separation of distribution centers}\\
\end{tabular}
\caption{Simulation with increasing separation of distribution centers.
Algorithms are the same as in \fref{fig:simu_dimension}. 
Only  second-order distance algorithms 
are performance invariant with   separation of distribution centers.
\label{fig:simu}}
\end{figure}

\paragraph{Real Images}
 with NetVLAD~\cite{arandjelovic2016netvlad} as image descriptors are used to evaluate clustering  on 5 data-sets: Handwritten numbers  in  Mnist~\cite{lecun1998gradient};
A mixture of images from  Google searches for ``Osaka castle''' and ``Christ the redeemer statue'';
2 sets of  10  object types from CalTech 101~\cite{griffin2007caltech}; And a mixture of ImageNet~\cite{deng2009imagenet} images from the Lion, Cat, Tiger classes.
Distribution-clustering is  evaluated against  five  baseline techniques: K-means ~\cite{Lloyd82leastsquares,arthur2007k}, spectral-clustering~\cite{zelnik2005self},
 projective-clustering~\cite{aggarwal1999fast}, GMM~\cite{bishop2007pattern} and quick-shift~\cite{vedaldi2008quick}. For k-means and GMM,  the number of clusters  is  derived from  distribution-clustering. This is typically  $20-200$.  Spectral and projective-clustering are prohibitively slow with many  clusters. Thus,
  their cluster numbers are fixed at $20$.

Cluster statistics are reported in \tref{tab:scores}.
On standard silhouette and purity scores,  distribution-clustering's  performance is comparable
to benchmark   techniques.    The performance  is  decent  for a new  approach and validates Theorem \ref{theorem:final}'s   ``contrast-loss''   constraint
in the real-world. However,  an interesting trend  hides in the average statistics.

Breaking down the purity score to find the percentage of images deriving from pure clusters, \ie, clusters with no wrong elements, we find that distribution-clustering assigns a remarkable fraction of images  to pure clusters. On average, it is $1.5$ times better than the  next best algorithm and in some cases can nearly double the performance.
This is important to data-abstraction where pure clusters allow a single average-feature to represent a set of features.
In addition, distribution-clustering ensures   pure clusters are readily identifiable.
\Fref{fig:clus_real} plots percentage error  as  clusters are processed in order of variance. Distribution-clustering keeps  ``outliers''  packed into  high-variance clusters, leaving  low-variance clusters especially pure. This  enables concepts like  image ``over-segmentation'' to be transferred to unorganized image  sets.

 K-means and GMM are the closest alternative to distribution-clustering. However, their clusters are   less  pure
and they are dependent on 
distribution-clustering to initialize the number of clusters. This makes distribution-clustering  one of the few (only?) methods  effective  on highly chaotic image data 
like Flickr11k~\cite{yang2008contextseer}  demonstrated in  \fref{fig:qual}.

\begin{figure}[t]
\centering
\setlength\tabcolsep{3pt} 
\begin{tabular}{cc}
\includegraphics[width=0.45\linewidth]{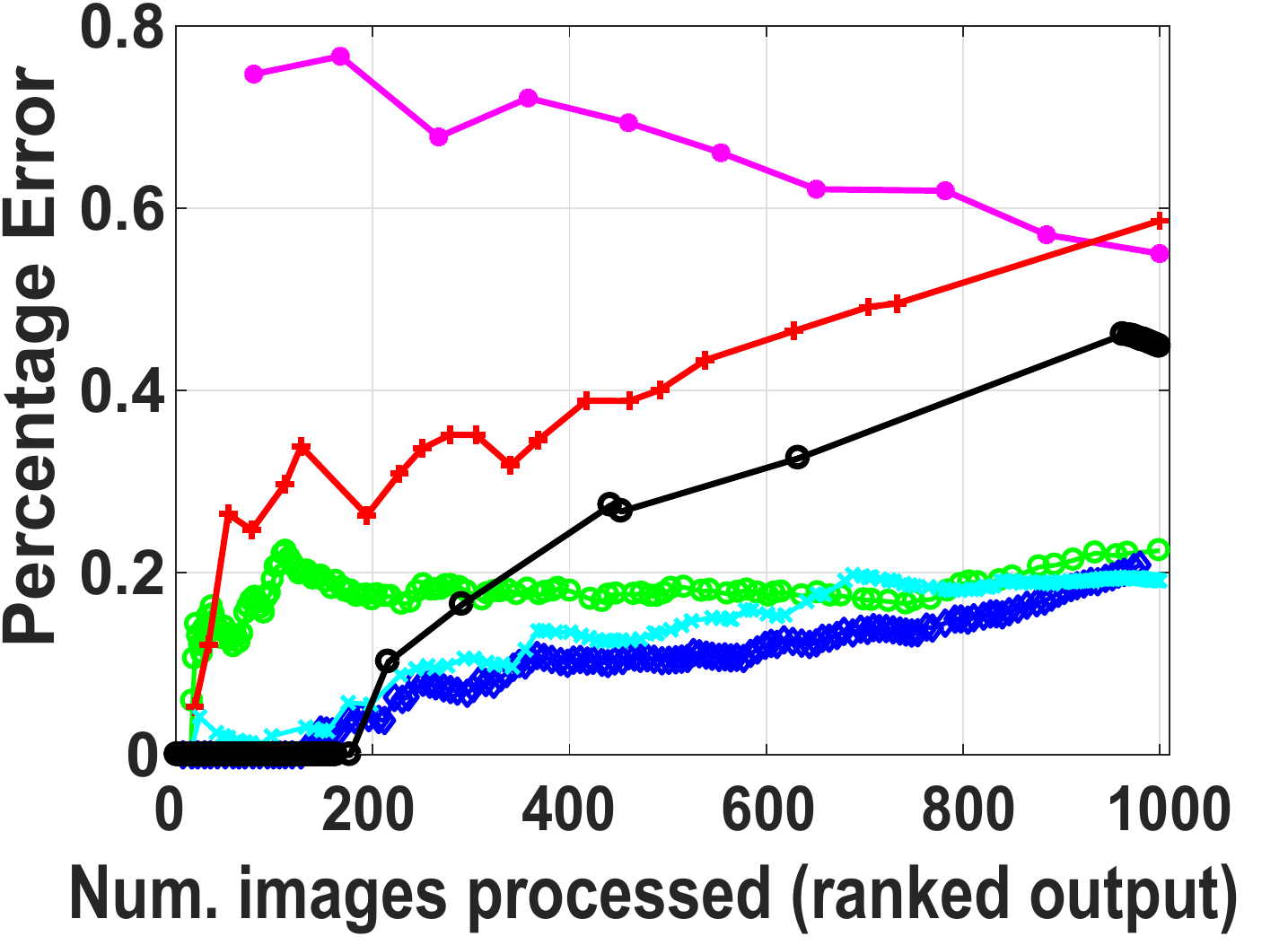}
&\includegraphics[width=0.45\linewidth]{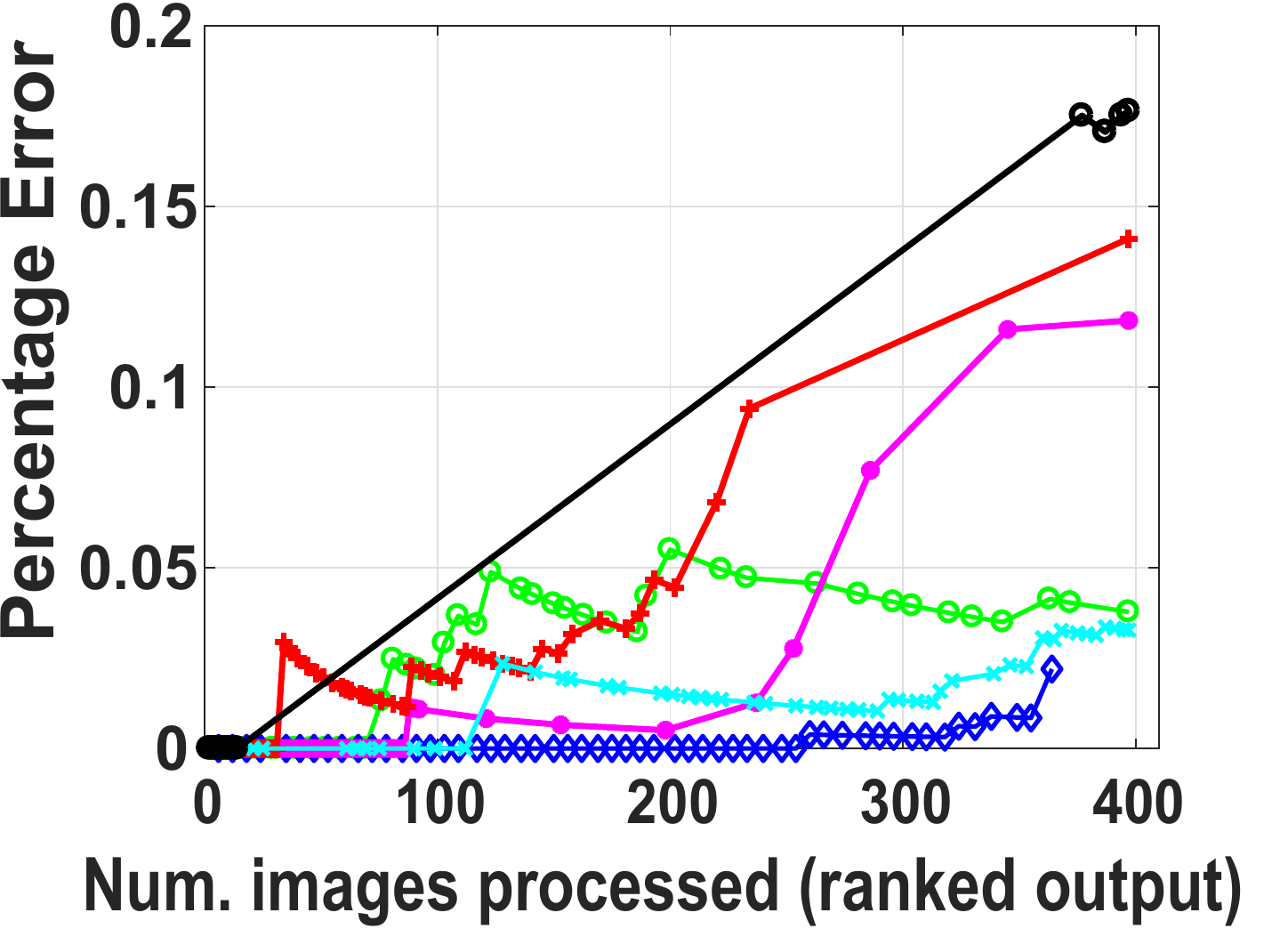}\\
Minst & Internet Images \\    \\
\includegraphics[width=0.45\linewidth]{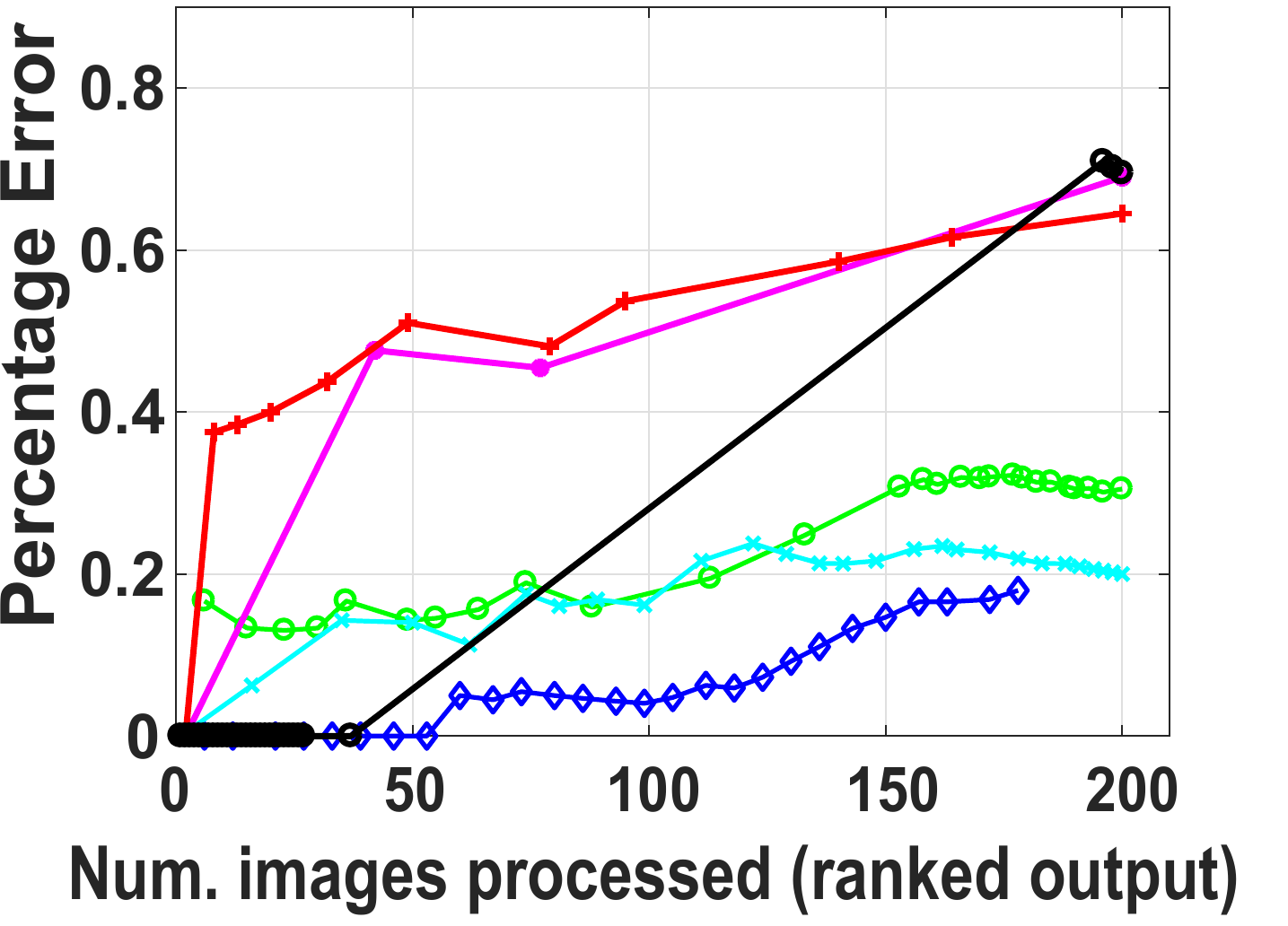}
&\includegraphics[width=0.45\linewidth]{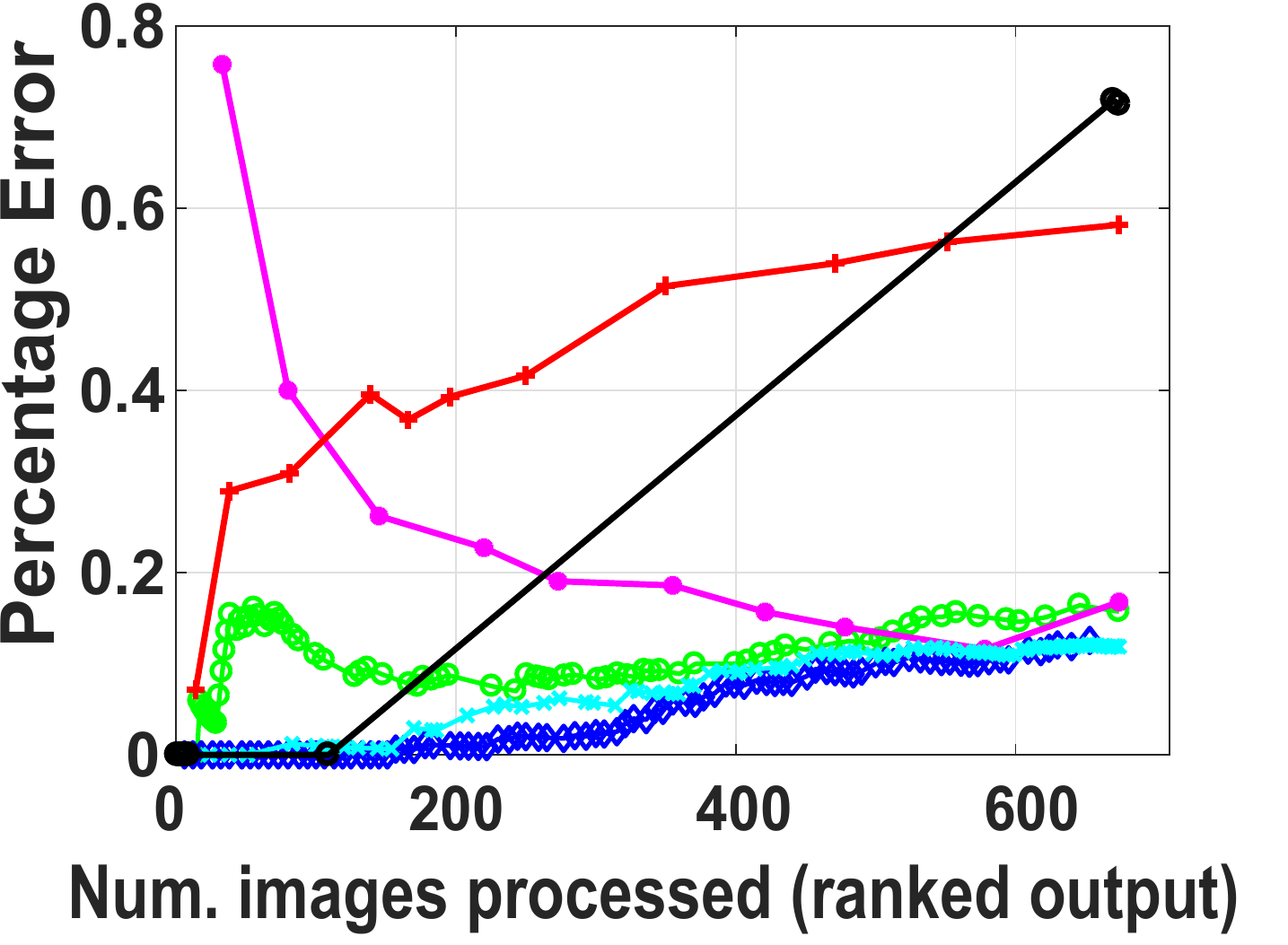}\\
CalTech101 (Set1) & CalTech101 (Set2) \\    \\
\includegraphics[width=0.45\linewidth]{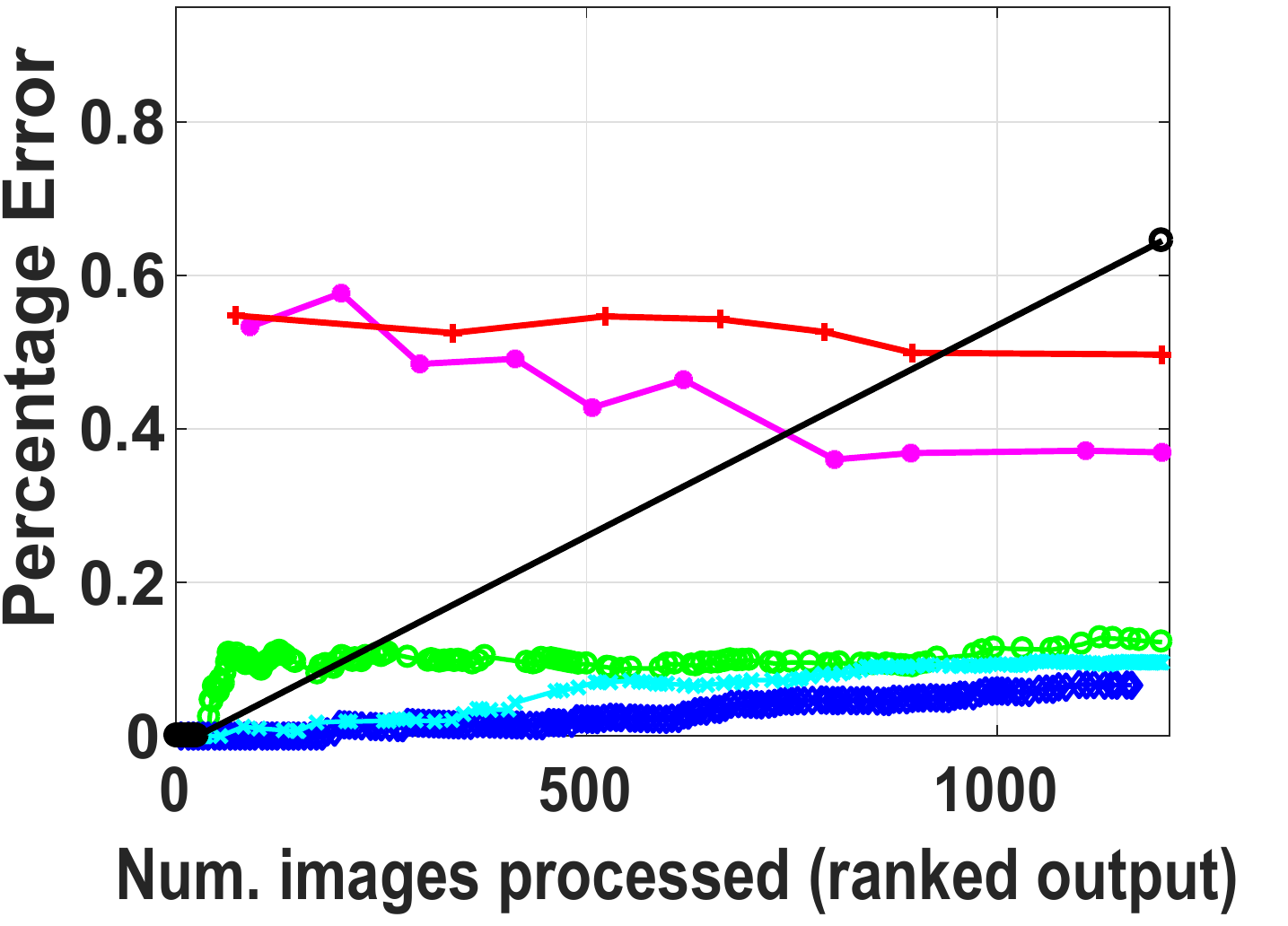}
&\includegraphics[width=0.2\linewidth]{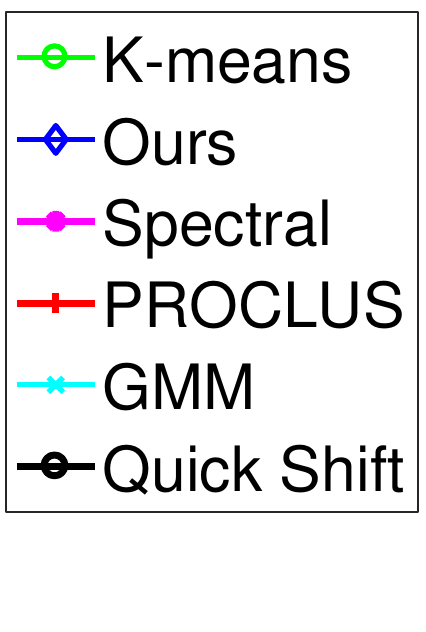}\\
Cats &  \\
\end{tabular}
\caption{Clustering real images. Clusters are ranked by variance. Unlike prior methods where ``outliers'' are randomly scattered in  clusters, distribution-clustering concentrates ``outliers'' in high variance clusters, making identification of pure clusters easy.  }
\label{fig:clus_real}
\end{figure}

\begin{table}[t]
\begin{center}
\setlength\tabcolsep{5pt}
\scriptsize
\begin{tabular}{|c|c|c|c|c|c|c|}
\hline
     &\multicolumn{6}{c|}{{\textit{Silhouette Score}}} \\
\hline
Dataset & K-Means & Spectral & PROCLUS & GMM& QS & Ours \\
        &~\cite{Lloyd82leastsquares,arthur2007k} &~\cite{zelnik2005self} & ~\cite{aggarwal1999fast} & \cite{bishop2007pattern} & \cite{vedaldi2008quick} &\\
\hline
MNIST       &0.0082  & 0.0267    & -0.0349       & 0.0076& \textbf{0.0730}&0.038 \\
Internet &0.084  & 0.041     & -0.038   &\textbf{ 0.0963}& 0.0488 & 0.003 \\
CalTech1     &0.027  & 0.02  & -0.0045           & 0.0373& \textbf{0.0999}& 0.074\\
CalTech2     &0.028  & \textbf{0.084}  & -0.0186 & 0.0248 & 0.0350& 0.042 \\
Cats     &0.007  & 0 & -0.002                    & 0.0236& \textbf{0.0752}& 0.0239 \\
Average & 0.0308& 0.034&-0.020                   &     0.0379
& \textbf{ 0.0532}& 0.036\\
\hline
     &\multicolumn{6}{c|}{ \textit{Purity Score} } \\
\hline
Dataset & K-Means & Spectral & PROCLUS & GMM & QS & Ours \\
\hline
MNIST       &0.77  & 0.45    & 0.41             & \textbf{0.81}& 0.55 & 0.79 \\
Internet    &0.96  & 0.90     & 0.86            &\textbf{0.97} & 0.82& \textbf{0.97} \\
CalTech1     &0.71  & 0.44     & 0.36           & 0.80&0.31 & \textbf{0.82} \\
CalTech2     &0.84  & 0.83     & 0.41           &\textbf{0.88} & 0.29& 0.87 \\
Cats     &0.88 & 0.63     &  0.50               & 0.90&0.35 & \textbf{0.93} \\

Average & 0.83& 0.65& 0.51                      &0.87 & 0.46&  \textbf{0.88}\\
\hline
     &\multicolumn{6}{c|}{\textit{$\%$ of images in pure clusters (excluding singletons)} } \\
\hline
Dataset & K-Means & Spectral & PROCLUS & GMM & QS & Ours \\
\hline
MNIST       &0.28  & 0      & 0                 &0.32 &0.059 & \textbf{0.49} \\
Internet    &0.83  & 0.82     & 0.40            &0.83 &0.031 & \textbf{0.92} \\
CalTech1     &0.08  & 0.02     & 0.01           &0.30 &0.081 & \textbf{0.52} \\
CalTech2     &0.47  & 0.24     & 0              &0.52 &0.16 & \textbf{0.65} \\
Cats     &0.34 & 0     & 0                      &0.40 &0.017 & \textbf{0.72} \\
Average & 0.40& 0.22&0.082                      &0.47 &0.070 &  \textbf{0.66}\\
\hline
     &\multicolumn{6}{c|}{\textit{$\%$ of  pure clusters (excluding singletons)} } \\
\hline
Dataset & K-Means & Spectral & PROCLUS & GMM & QS  &Ours \\
\hline
MNIST       &0.43  & 0      & 0        &0.44 &\textbf{0.62} & 0.49 \\
Internet    &0.80  & 0.90     & 0.8    &0.83 &0.40 & \textbf{0.93} \\
CalTech1     &0.30  & 0.20     & 0.09  &0.46 & \textbf{0.75} & 0.52 \\
CalTech2     &0.54  & 0.20     & 0     &0.57 &\textbf{0.67} & \textbf{0.67} \\
Cats     &0.57 & 0     & 0             &0.55 &0.50 & \textbf{0.74} \\
Average & 0.53& 0.32&0.178             &0.57 &0.59 &  \textbf{0.67}\\
\hline
\end{tabular}
\end{center}
\caption{Cluster Statistics. Distribution-clustering ensures a large percentage of images belong to pure clusters.  }\label{tab:scores}
\end{table}

\paragraph{Timing} excludes feature extraction cost which is common to all algorithms. Experiments are on  
an  i7  machine, with $4096$ dimension NetVlad~\cite{arandjelovic2016netvlad} features computed over the $400$ images of Internet data-set.   
Our single core,  Matlab implementation of  distribution-clustering takes $23$ seconds, of which $4.5$ seconds was spent computing the affinity matrix. Timing for other algorithms are as follows. 
 K-means~\cite{Lloyd82leastsquares,arthur2007k}: $0.73$ seconds, Quick Shift~\cite{vedaldi2008quick} (Python): $1.5$ seconds,  spectral clustering~\cite{zelnik2005self}  (20 clusters): $2$ seconds, GMM\footnote{GMM's timing is with  covariance estimation.
Fixed covariance matrix permits   convergence in   seconds but is  inappropriate on some data.} \cite{bishop2007pattern}: $4$ minutes  and  PROCLUS~\cite{aggarwal1999fast} (20 clusters on OpenSubspace V3.31):  9 minutes.

\paragraph{Qualitative} inspection of  distribution-clusters show they have a purity not captured by  quantitative evaluation.  \Fref{fig:clus} illustrates this  on  Colosseum images crawled from the web. Quantitatively,   both distribution-clustering and k-means are 
nearly equal, with few clusters  mixing  Colosseum and ``outlier'' images. However,  distribution-clusters  are qualitatively better, with images  in  a cluster sharing a clear, generative distribution.  

\paragraph{Other} things  readers may want to note. More qualitative evaluation of clustering  is available in the supplementary. Code is available at \textcolor{blue}{http://www.kind-of-works.com/}.

\begin{figure}
\center
\begin{tabular}{c|c}
\includegraphics[width=0.45\linewidth]{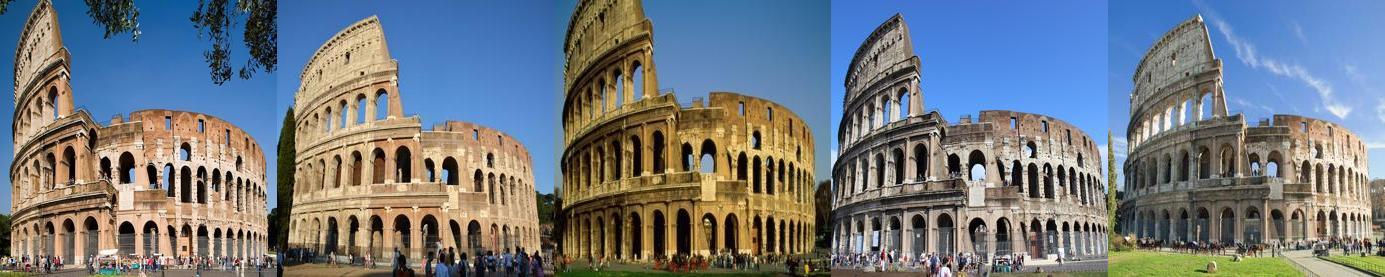}&
\includegraphics[width=0.45\linewidth]{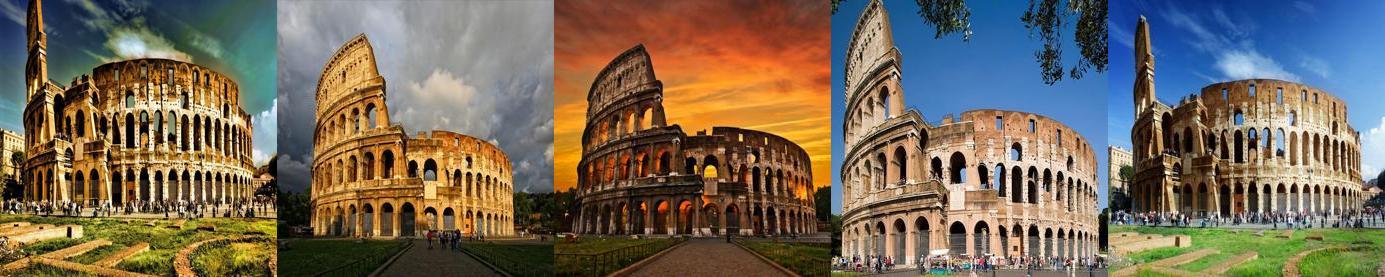}\\
\includegraphics[width=0.45\linewidth]{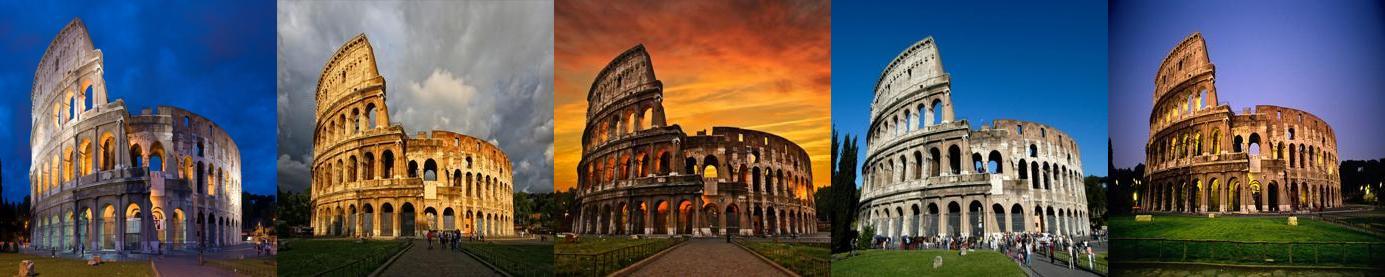}&
\includegraphics[width=0.45\linewidth]{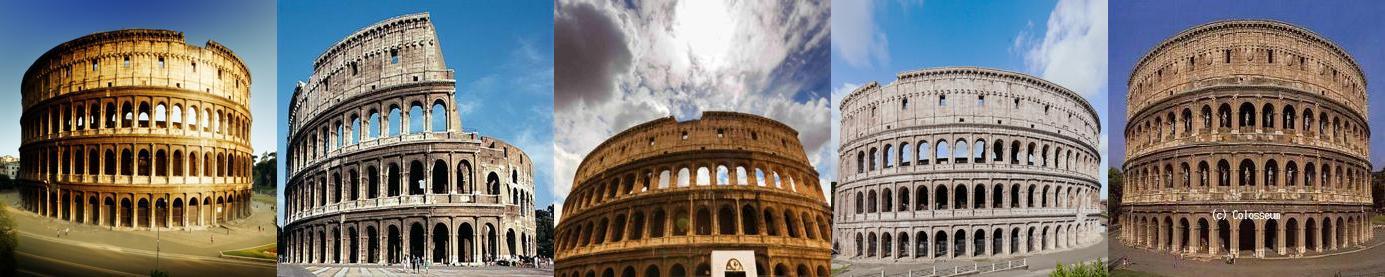}\\
\includegraphics[width=0.45\linewidth]{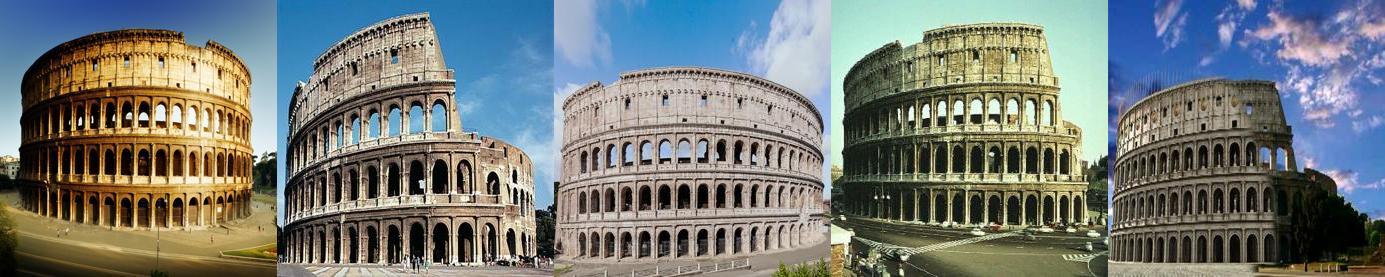}&
\includegraphics[width=0.45\linewidth]{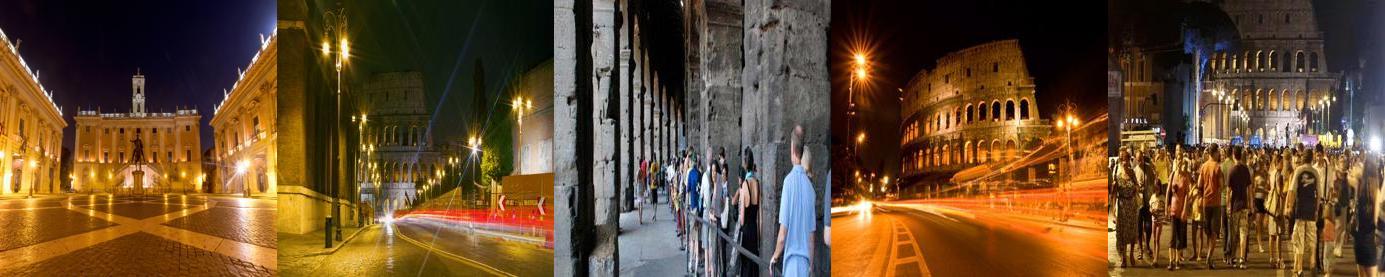}\\
\includegraphics[width=0.45\linewidth]{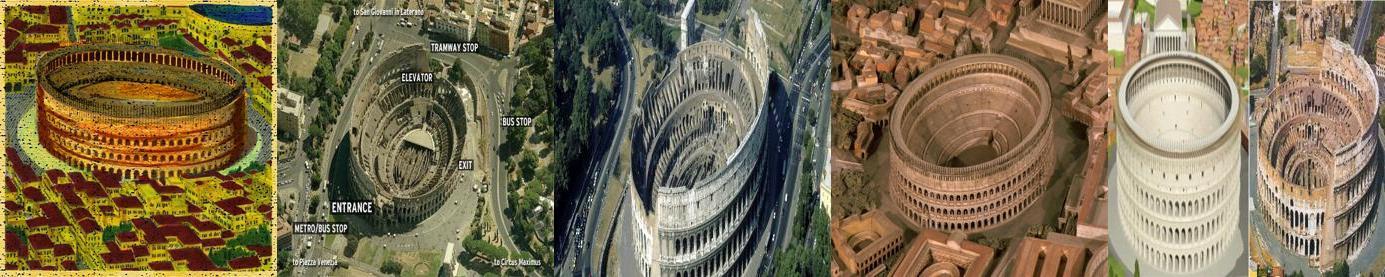}&
\includegraphics[width=0.45\linewidth]{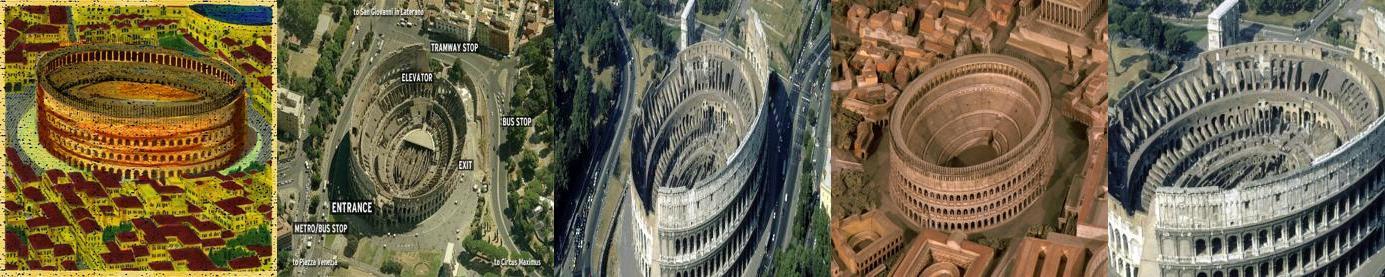}\\
\includegraphics[width=0.45\linewidth]{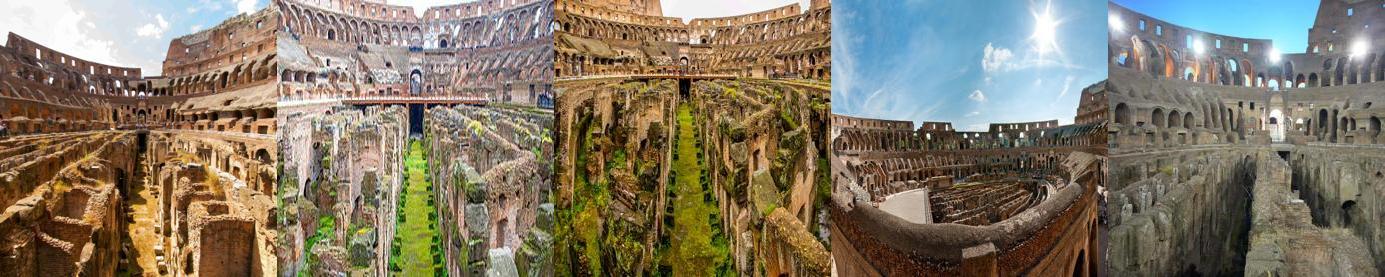}&
\includegraphics[width=0.45\linewidth]{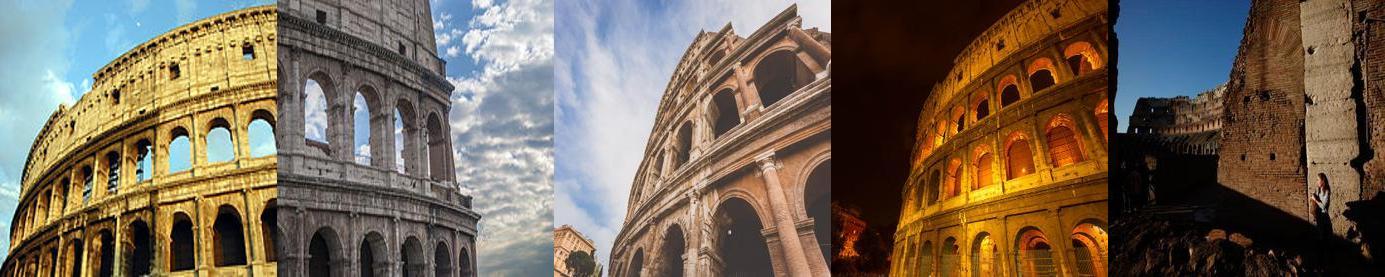}\\
\includegraphics[width=0.45\linewidth]{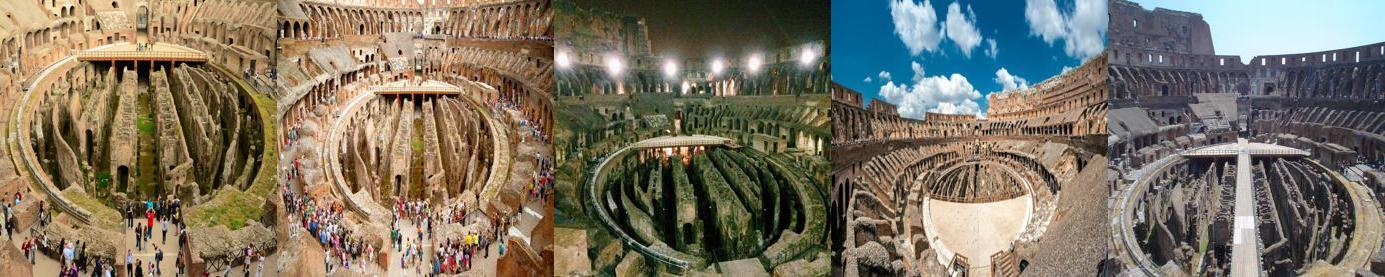}&
\includegraphics[width=0.45\linewidth]{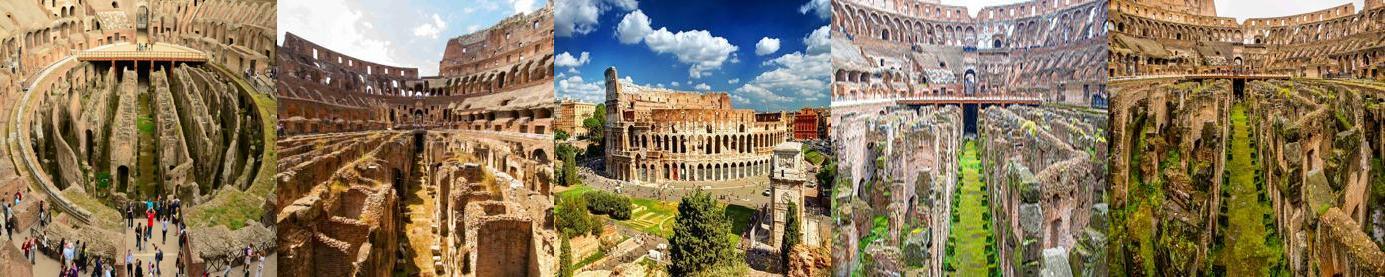}\\
\includegraphics[width=0.45\linewidth]{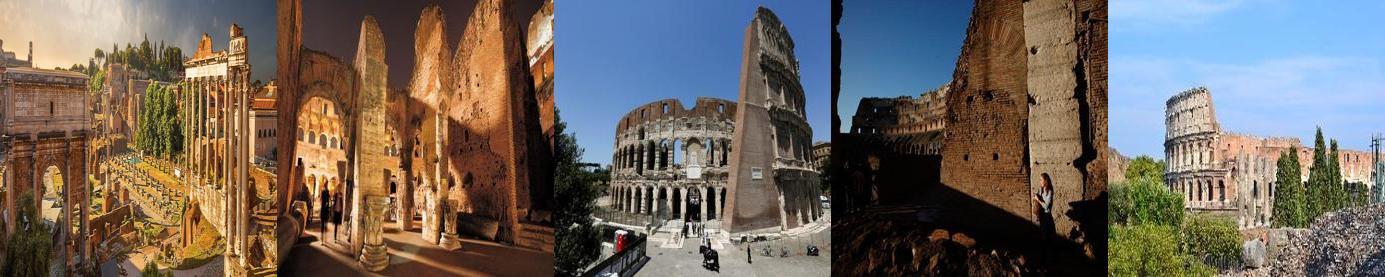}&
\includegraphics[width=0.45\linewidth]{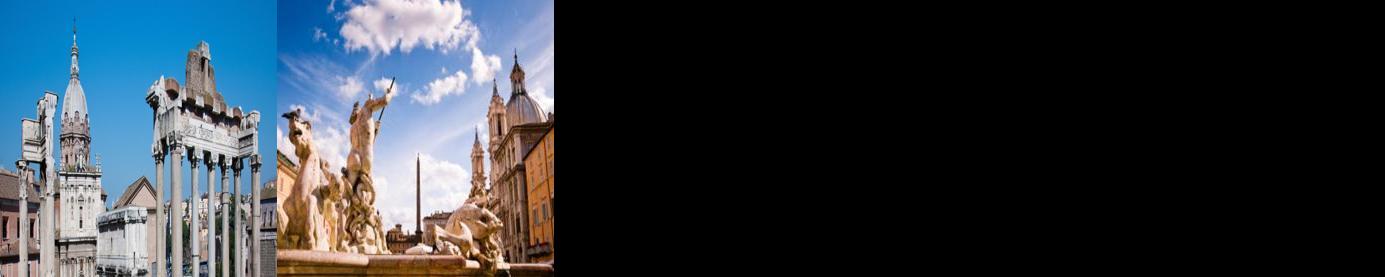}\\
\includegraphics[width=0.45\linewidth]{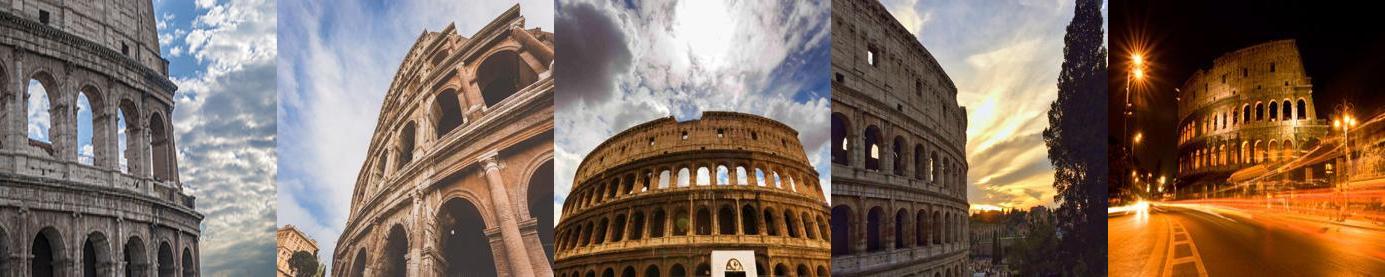}&
\includegraphics[width=0.45\linewidth]{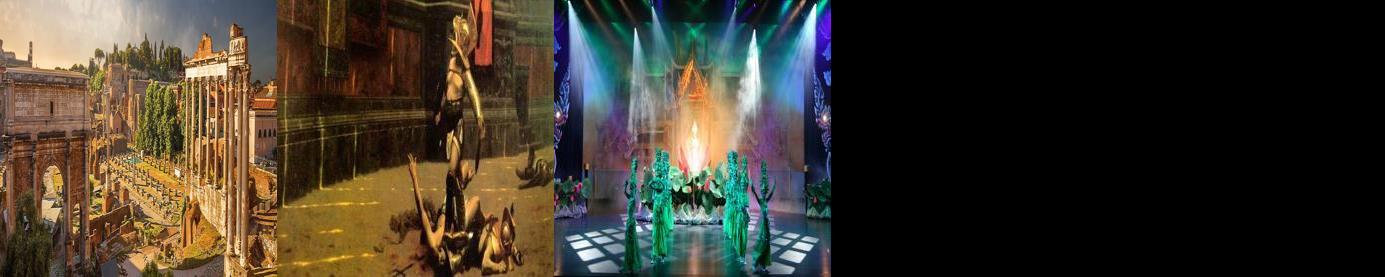}\\
\includegraphics[width=0.45\linewidth]{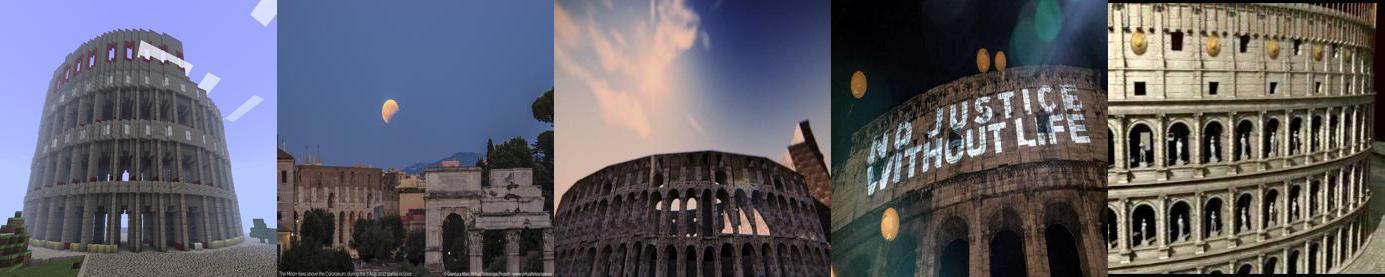}&
\includegraphics[width=0.45\linewidth]{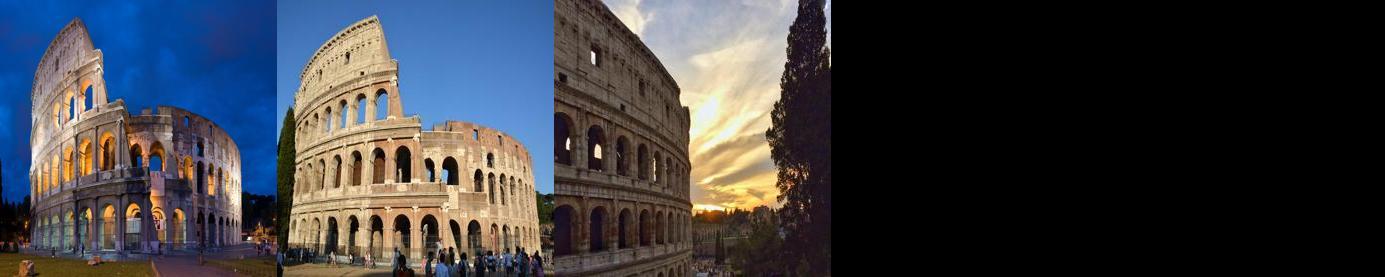}\\
 Distribution-clustering & \emph{K-means} clustering
\end{tabular}
\caption{Distribution-clustering provides  fine  intra-cluster consistency, with  images of a cluster sharing a clear, generative distribution. This  qualitative improvement is not captured in evaluation statistics.
\label{fig:clus}
 }
\end{figure}

\section{Conclusion}

We have shown that chaotically overlapping distributions become
intrinsically separable in high dimensional space and proposed a distribution-clustering algorithm
to achieve it. 
By turning a former curse of dimensionality into a blessing, distribution-clustering  is a powerful technique for  discovering  patterns and trends
in  raw data. This   can  impact  a wide range of disciplines ranging from  semi-supervised learning to  bio-informatics.  

\section*{Acknowledgment}
This paper is dedicated to Professor Douglas L. Jones, Professor Minh N. Do and Hong-Wei Ng.

Y. Matsushita is supported by JST CREST Grant Number JP17942373, Japan. 


\end{document}